\documentclass{article}
\usepackage{PRIMEarxiv}

\usepackage[utf8]{inputenc}
\usepackage[T1]{fontenc}
\usepackage{hyperref}
\usepackage{url}
\usepackage{natbib} 
\usepackage{booktabs}
\usepackage{amsfonts}
\usepackage{amssymb}
\usepackage{nicefrac}
\usepackage{microtype}
\usepackage{graphicx}
\usepackage{multirow}
\usepackage{array}
\usepackage{siunitx}
\usepackage{xcolor}
\usepackage{listings}

\usepackage{amsmath}
\usepackage{amsthm}
\usepackage{bm}
\usepackage{bbm}
\usepackage{pifont}
\usepackage{enumitem}
\usepackage{caption}
\usepackage{subcaption}
\usepackage{tabularx}
\usepackage{colortbl}
\usepackage{algorithm}
\usepackage{algpseudocode}

\newtheorem{theorem}{Theorem}
\newtheorem{lemma}{Lemma}
\newtheorem{proposition}{Proposition}

\newtheorem{definition}{Definition}

\newcommand{\mansu}{\textsc{Mansu}}
\newcommand{\Df}{\mathcal{D}_f}
\newcommand{\Dr}{\mathcal{D}_r}
\newcommand{\C}{\mathcal{C}}
\newcommand{\Cbar}{\bar{\mathcal{C}}}
\newcommand{\Hmat}{\mathbf{H}}
\newcommand{\Fmat}{\mathbf{F}}
\newcommand{\dtheta}{\Delta\theta}
\newcommand{\norm}[1]{\left\|#1\right\|}

\newcommand{\ptqgap}{\Delta_{\mathrm{PTQ}}}
\newcommand{\res}[1]{\textbf{\textcolor{blue}{#1}}}
\definecolor{darkgreen}{rgb}{0.0,0.42,0.0}

\newcommand{\cmark}{\textcolor{darkgreen}{\ding{51}}}
\newcommand{\xmark}{\textcolor{red}{\ding{55}}}

\definecolor{codegreen}{rgb}{0,0.6,0}
\definecolor{codegray}{rgb}{0.5,0.5,0.5}
\definecolor{codepurple}{rgb}{0.58,0,0.82}
\definecolor{backcolour}{rgb}{0.97,0.97,0.97}

\title{\includegraphics[width=0.99\textwidth]{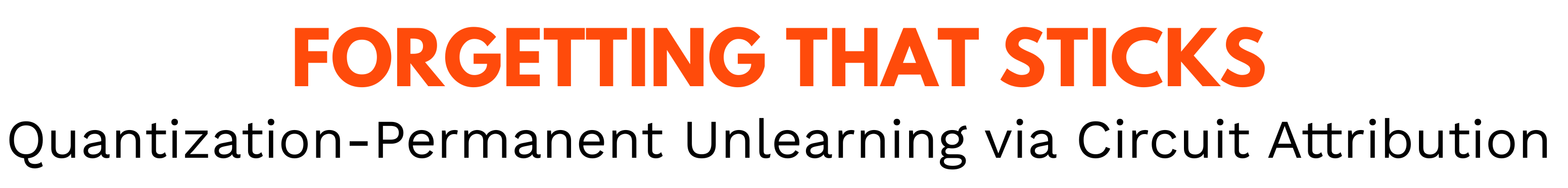}}

\author{
  Saisab Sadhu, Pratinav Seth, Vinay Kumar Sankarapu \\
  \affiliation{Lexsi Labs} \\
  \texttt{saisab.sadhu@lexsi.ai}
}

\setabstract{
Standard unlearning evaluations measure behavioral suppression in full
precision, immediately after training, despite every deployed language
model being quantized first. Recent work has shown that 4-bit
post-training quantization can reverse machine unlearning; we show this
is not a tuning artefact but a systematic \textbf{dual failure}: gradient-based
methods that achieve meaningful forgetting \emph{lose it under compression},
while methods that survive quantization barely change the model. Both
failures trace to the same root cause: across all baselines, per-parameter
updates lie 47–828× below the NF4 quantization bin width updates
diffused across billions of parameters cannot clear quantization bin
boundaries a consequence we formalize as a sparsity-permanence tradeoff.
We present \textbf{\mansu{}} (\textbf{M}echanistic-\textbf{A}ligned
\textbf{N}ull-\textbf{S}pace \textbf{U}nlearning), which resolves both
modes by combining causal circuit attribution to isolate the minimal
forget-set subgraph, circuit-restricted null-space projection with a
diagonal-Fisher retain bound, and a per-parameter magnitude floor
guaranteeing \emph{quantization survival by construction}. We additionally
introduce \textbf{Circuit Attribution Divergence (CAD)}, a mechanistic
verification metric distinguishing \emph{structural erasure} from
\emph{behavioral suppression}, a distinction existing metrics cannot make.
Across multiple model families and hazard benchmarks, \mansu{} is the
first method to jointly satisfy all four properties with margin on each
(meaningful forgetting, retain preservation, non-positive PTQ gap, and
structural erasure), while gradient-based baselines recover up to
$+0.05$ accuracy under compression.
}

\setkeywords{Machine Unlearning, Large Language Models, Model Quantization, Circuit Attribution, AI Safety}

\runningtitle{Quantization-Permanent Unlearning via Circuit Attribution}

\begin{document}
\maketitle

\section{Introduction}\label{sec:intro}
\label{sec:introduction}

Machine unlearning has become a safety-critical capability for
deployed language models: hazardous-knowledge memorisation
(biosecurity, cyberweapons, chemical synthesis) makes it
necessary~\citep{li2024wmdp}, and right-to-erasure regulations (EU AI
Act, GDPR) make it legally required~\citep{jang2023forget}.
Yet every deployed LLM today is quantized 4-bit formats (NF4,
GPTQ, AWQ) reduce memory by $4\times$ and inference cost by
$2\text{–}3\times$, making quantization the standard final step before
release. \citep{zhang2024catastrophic} (ICLR 2025) documented that
4-bit PTQ can reverse machine unlearning, reporting up to $83\%$
recovery and proposing a saliency-based mitigation (PTQ-LR/SURE);
standard evaluation practice has not yet caught up the field's
default protocol remains
behavioral metrics in full precision on a held-out forget set, measured
immediately after training. We trace the reversal phenomenon to a
structural cause (per-parameter updates systematically fall below the
NF4 bin width) and propose a method that addresses it by construction.
The assumption the standard protocol embeds behavioral suppression
in BF16 is an adequate proxy for durable knowledge removal is false,
and the failure is systematic.

\textbf{The dual failure mode:}
We apply six representative methods to Llama-3.1-8B-Instruct on
WMDP-bio~\citep{li2024wmdp} and confirm that gradient-based methods
achieve meaningful forget-set suppression in BF16. We then apply NF4
4-bit post-training quantization (the compression scheme used by the
overwhelming majority of real-world LLM deployments) and re-evaluate.
In every gradient-based method, \emph{the forgotten knowledge returns},
with PTQ recovery gaps of $+0.06$ to $+0.07$. Methods that survive
quantization do so only by barely changing the model: across $94$
non-\mansu{} experiments (Table~\ref{tab:app_sweep_summary}),
preference-optimization and null-space methods reduce forget-set
accuracy by $1.6$ pp on average, within measurement variance on a
four-way MCQ task.
The pattern holds on Qwen-3-8B and on MUSE open-ended memorization,
ruling out model- or benchmark-specific explanations.

\textbf{The structural cause:}
Both failure modes share one origin. Every existing method distributes
gradient updates across all $d$ parameters. For Llama-3.1-8B
($d \approx 8\times10^9$), even a large-norm gradient induces
per-parameter changes of order $10^{-6}$, far below the NF4
quantization bin width of $\approx 8.4\times10^{-4}$. At compression
time, these changes round to zero. Methods that avoid this by
constraining updates to remain near the original model do so at the
cost of meaningful forgetting. This is \emph{not a hyperparameter problem};
it is a necessary consequence of applying any gradient-based objective
uniformly across billions of parameters
(Proposition~\ref{prop:tradeoff}).

\textbf{The fix:}
Mechanistic interpretability has established that specific factual
knowledge is causally localized in sparse, identifiable subgraphs of
the model's
computation~\citep{meng2022rome,elhage2022superposition,syed2023attribution}.
If knowledge resides in $|\C|$ parameters rather than all $d$,
concentrating updates into $\C$ amplifies per-parameter magnitudes by
$\sqrt{d/|\C|}$. With an explicit magnitude floor, \emph{quantization survival becomes a
construction-time guarantee}. Null-space projection restricted
to $\C$ yields a retain-set loss bound provably tighter than global
projection by the Cauchy interlace theorem.

We present \mansu{} (\textbf{M}echanistic-\textbf{A}ligned
\textbf{N}ull-\textbf{S}pace \textbf{U}nlearning), which operationalizes
this insight: (1)~EAP-IG~\citep{hanna2024faithfulness} identifies the
minimal circuit $\C$ causally responsible for forget-set answers;
(2)~gradient updates within $\C$ are projected into the null space of
the retain-set Fisher Information, with a tighter bound proved in
Theorem~\ref{thm:retain}; and (3)~every cumulative update below the NF4
bin size is rescaled to the floor, guaranteeing quantization survival by
construction (Lemma~\ref{lem:quant}). On Llama-3.1-8B-Instruct /
WMDP-bio, \mansu{} achieves a PTQ gap of $-0.040$ while preserving MMLU
within $0.030$ of the zero-shot model: NF4 \emph{amplifies} rather than
reverses the erasure (Proposition~\ref{prop:amplification}).
Results replicate on Qwen-3-8B and MUSE.

\textbf{Contributions:}
(I)~\textbf{Dual failure documentation:} the first systematic evidence
that no existing method achieves both meaningful forgetting and
quantization permanence, across $94$ non-\mansu{} experiments
($84$ WMDP cells from the family-wise sweep in
Table~\ref{tab:app_sweep_summary} plus $10$ MUSE cells in
Table~\ref{tab:main}) over three model families, three hazard domains,
and two benchmarks.
(II)~\textbf{\mansu{}}: a three-component method with formal guarantees,
tighter retain bound (Theorem~\ref{thm:retain}), construction-time
quantization survival (Lemma~\ref{lem:quant}), and a sparsity-permanence
tradeoff analysis (Proposition~\ref{prop:tradeoff}); full proofs in
Appendix~\ref{app:proof}.
(III)~\textbf{Circuit Attribution Divergence (CAD):} the first post-hoc
mechanistic verification protocol distinguishing structural knowledge
deletion from behavioral suppression, a distinction standard behavioral
metrics cannot make (Section~\ref{sec:cad}).


\section{Background and Related Work}
\label{sec:related}
\begin{table}[pt]
\caption{\textbf{Prior work against the four unlearning requirements}
(Section~\ref{sec:problem}).
\textbf{F}~=~forget; \textbf{R}~=~retain; \textbf{Q}~=~quant.-permanent
($\ptqgap \leq 0$); \textbf{S}~=~structural erasure (CAD~$\gg 0$).
\cmark~satisfies; \xmark~fails; $\sim$~partial.
Methods($\dagger$) are evaluated in our experiments.}
\label{tab:related}
\centering
\footnotesize
\setlength{\tabcolsep}{4pt}
\begin{tabularx}{\linewidth}{@{}llccccX@{}}
\toprule
Method & Family & \textbf{F} & \textbf{R} & \textbf{Q} & \textbf{S} & Key limitation \\
\midrule
GA$^\dagger$~\citep{jang2023forget}
  & Grad.\ ascent & \cmark & $\sim$ & \xmark & \xmark
  & Updates diffuse over $d$ params; per-param $\!\ll\! \delta_i$ \\
Surgical GA$^\dagger$~\citep{jang2023forget}
  & Grad.\ ascent & \cmark & $\sim$ & \xmark & \xmark
  & Layer restriction reduces diffusion but cannot reach $\delta_i$ \\
NPO$^\dagger$~\citep{zhang2024npo}
  & Pref.\ opt. & $\sim$ & \cmark & \cmark & \xmark
  & Frozen reference prevents updates $\geq \delta_i$ \\
SimNPO$^\dagger$~\citep{liu2024simnpo}
  & Pref.\ opt. & \xmark & \cmark & \cmark & \xmark
  & Same frozen-anchor problem as NPO \\
GU+SimNPO$^\dagger$~\citep{huang2024unified,liu2024simnpo}
  & Null-space + Pref. & \xmark & \cmark & \cmark & \xmark
  & Global projection reinstates diffusion problem \\
LUNAR$^\dagger$~\citep{lunar2025}
  & Repr.\ steer. & $\sim$ & \cmark & \xmark & \xmark
  & Edits non-circuit MLP projection; PTQ gap $\approx+0.01$ (Table~\ref{tab:main}) \\
PTQ-LR~\citep{zhang2024catastrophic}
  & Quant.-aware & \cmark & $\sim$ & $\sim$ & \xmark
  & Raises LR; retain constraint re-bounds max useful LR \\
\midrule
\mansu{} (ours)
  & Circuit + floor & \cmark & \cmark & \cmark & \cmark
  & First method jointly satisfying all four with margin on each \\
\bottomrule
\end{tabularx}
\end{table}

Machine unlearning methods can be grouped into five method families
(gradient ascent, preference optimization, null-space projection,
representation steering, quantization-aware optimization);
Table~\ref{tab:related} summarizes each against the four requirements
of Section~\ref{sec:problem}.

\textbf{Gradient ascent and variants}~\citep{jang2023forget,liu2022continual}
maximize forget-set loss directly. These methods are simple and
effective in full precision, but updates distribute over all $d$
parameters, pushing per-parameter magnitudes far below quantization bin
widths. Surgical variants~\citep{jang2023forget} reduce the active
parameter count but cannot reach the bin threshold without violating
the retain constraint (Proposition~\ref{prop:tradeoff}).

\textbf{Preference optimization} (NPO~\citep{zhang2024npo},
SimNPO~\citep{liu2024simnpo}) adapts DPO~\citep{rafailov2024dpo} to
treat forget-set responses as dis-preferred. The frozen reference model
prevents output collapse and incidentally prevents large per-parameter
updates, giving good retain scores but negligible structural change.
TOFU~\citep{maini2024tofu} and MUSE~\citep{shi2024muse} are
benchmark suites for preference-optimized unlearning, on
fictitious-author facts and open-ended memorization respectively; we
evaluate on MUSE alongside the WMDP hazard splits.

\textbf{Null-space projection}~(GU,~\citealp{huang2024unified}) projects
gradient updates onto the null space of the retain Hessian, giving a
formal retain-safety bound. Because the projection is \emph{global},
the diffusion problem is reinstated. \mansu{} inherits the projection
idea and proves a strictly tighter bound by restricting both the update
and the projection to the causally identified circuit
(Theorem~\ref{thm:retain}).

\textbf{Representation steering} (LUNAR~\citep{lunar2025},
RMU~\citep{li2024wmdp}) suppresses forget-set outputs by redirecting
activations at inference time. LUNAR trains only a single MLP
down-projection outside the EAP-IG forget circuit; RMU randomises
forget-set activations without weight edits. In both cases the causal
knowledge circuit is left intact, so the unlearned model passes
behavioural metrics while CAD remains $\approx 0$ the failure mode
CAD is designed to expose. We include LUNAR in our experiments and
discuss RMU as a methodologically adjacent baseline.

\textbf{Quantization robustness:}
\citep{zhang2024catastrophic} (ICLR 2025) document that 4-bit PTQ can
catastrophically reverse unlearning, reporting up to $83\%$ recovery
and proposing a saliency-based unlearning strategy with a large
learning rate (``PTQ-LR'' in Table~\ref{tab:related}) as mitigation.
We show (Proposition~\ref{prop:tradeoff}) that the retain constraint
independently caps the useful learning rate, so the root cause
remains unaddressed. Our magnitude-floor constraint solves the
problem at its source.

\textbf{Mechanistic interpretability and knowledge localization:}
ROME~\citep{meng2022rome} and MEMIT~\citep{meng2023memit} established
via causal patching that factual associations are stored in middle MLP
layers; EAP-IG~\citep{hanna2024faithfulness} extends this to
circuit-level attribution across the full computation graph.
Concurrently, \citep{kasliwal2026circuit} apply circuit-restricted
weight arithmetic to embed \emph{refusal} directly into checkpoints
without inference-time hooks. Our work applies the same localization
principle to \emph{unlearning} and adds the orthogonal constraint of
quantization permanence, which that setting does not require.
\citep{lee2025negative} and \citep{guo2025mechanistic} raise concerns
that attribution-based circuits do not reliably predict unlearning
targets; Ablation~C(i) tests this claim directly on the
factual-recall benchmarks studied here and finds a substantial CAD
advantage ($1.143$ vs $0.743$) for the causally identified circuit
over a random same-size baseline at matched forget depth. Extended
discussion is in Appendix~\ref{app:related}.

\section{Problem Formulation}
\label{sec:problem}
\begin{figure}[!t]
\centering
\begin{minipage}[t]{0.27\textwidth}
  \centering
  \includegraphics[width=\linewidth]{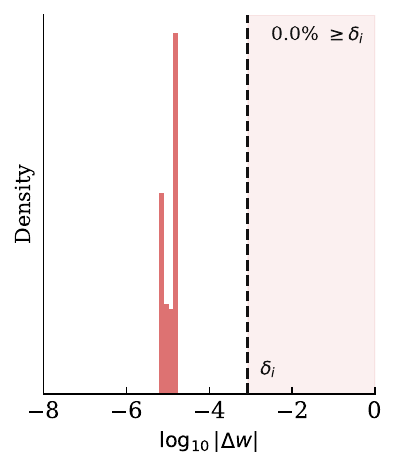}
\end{minipage}
\hfill
\begin{minipage}[t]{0.27\textwidth}
  \centering
  \includegraphics[width=\linewidth]{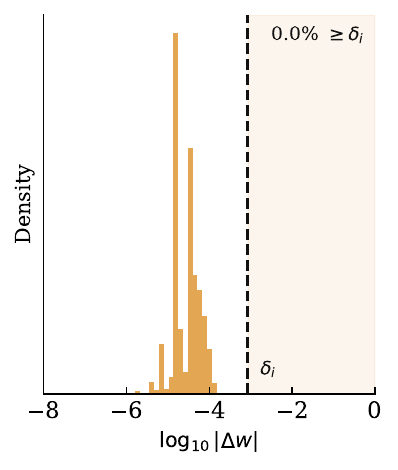}
\end{minipage}
\hfill
\begin{minipage}[t]{0.27\textwidth}
  \centering
  \includegraphics[width=\linewidth]{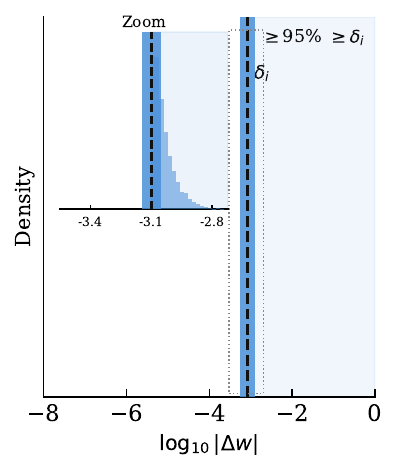}
\end{minipage}
\caption{\textbf{Per-parameter update magnitudes} (Llama-3.1-8B / WMDP-bio).
Histograms of $\log_{10}|\Delta w|$.
\textbf{(a)}~Global GA: diffuse, far below $\delta_i$.
\textbf{(b)}~Surgical GA: concentrated on L14–16, still below $\delta_i$.
\textbf{(c)}~\mansu{}: clamped at or above $\delta_i$ by construction.
Dashed line: NF4 bin width $\delta_i = 8.4\!\times\!10^{-4}$; updates to
its left round to zero under 4-bit quantization, so only \mansu{}'s
erasure survives (Lemma~\ref{lem:quant}).}
\label{fig:main_panel}
\end{figure}

Let $\theta \in \mathbb{R}^d$ be a pretrained LM's parameters, $\Df$ the forget set, $\Dr$ the retain set.

We seek $\dtheta$ with $\theta' = \theta + \dtheta$ satisfying four properties:
\textit{(i)~forget}: $\theta'$ fails on $\Df$ by a meaningful margin;
\textit{(ii)~retain}: performance on $\Dr$ and general benchmarks
within 2 pp of $\theta$;
\textit{(iii)~quantization permanence}: $Q_4(\theta')$ also fails on
$\Df$, where $Q_4$ is the deployment 4-bit quantizer;
\textit{(iv)~structural erasure}: re-running causal attribution on
$\theta'$ shows the subgraph implementing forget-set knowledge has collapsed, not merely been bypassed.
Properties (i) and (ii) are standard; (iii) and (iv) are not, and no existing method satisfies both.

\begin{definition}[NF4 quantization floor]
\label{def:floor}
Under NF4 quantization~\citep{dettmers2023qlora} with per-channel scale
$s_i$ and codebook levels $\{q_k\}_{k=0}^{15}$, the smallest bin
width for parameter $i$ is
$\delta_i = s_i \cdot \min_{k} |q_k - q_{k-1}|$.
For Llama-3.1-8B MLP weights $\delta_i \approx 8.4\times10^{-4}$
(derivation in Appendix~\ref{app:nf4}).
\end{definition}

\begin{proposition}[Sparsity–permanence tradeoff]
\label{prop:tradeoff}
Under gradient ascent with retain constraint
$\mathcal{L}_r(\theta+\dtheta)-\mathcal{L}_r(\theta)\leq\epsilon_r$,
the per-parameter update magnitude when $|\C|$ parameters are updated
(all others frozen) satisfies
\begin{equation}
\norm{\dtheta_i}
\;\leq\;
\sqrt{\frac{2\epsilon_r}{|\C|\,\bar{F}_\C}},
\qquad \bar{F}_\C \;=\; \frac{1}{|\C|}\sum_{j\in\C}[\Fmat_r]_{jj},
\label{eq:diffusion}
\end{equation}
where $[\Fmat_r]_{jj}=\mathbb{E}_{(x,y)\sim\Dr}[(\partial\log p_\theta(y|x)/\partial\theta_j)^2]$
is the empirical diagonal Fisher of the retain loss
(Appendix~\ref{app:proof}; the diagonal Fisher remains well-defined
under rank-deficient $\Hmat_r$, unlike the standard $\sigma_{\min}$ form).
For Llama-3.1-8B ($d=8.03\!\times\!10^9$, $\epsilon_r=0.02$,
$\bar{F}_\C\!\sim\!10^{0}$), the global case
($|\C|=d$) gives
$\norm{\dtheta_i}\!\lesssim\!2.2\!\times\!10^{-6}$, roughly
$380\times$ below $\delta_i$. Updates reach $\delta_i$ only when
$|\C|/d \leq 7\!\times\!10^{-6}$ (fewer than $0.001\%$ of parameters).
\end{proposition}

\textbf{Implications.}
\emph{First}, no existing gradient-based method operates near this
threshold: Surgical GA's $6.6\%$ circuit and even \mansu{}'s $3.2\%$
both sit more than three orders of magnitude above it
($\approx\!4500\times$ for \mansu{}, $\approx\!9400\times$ for Surgical
GA; cf.\ Surgical GA's $+0.027$ PTQ gap, Table~\ref{tab:main}), so
localization alone is insufficient and the magnitude floor
(Section~\ref{sec:method}) is required to close the gap by construction.
\emph{Second}, Proposition~\ref{prop:tradeoff} says nothing about
\emph{which} parameters to update: arbitrary concentration damages
retain performance, so the circuit must be chosen causally.

\textbf{Second failure mode.}
Preference-optimization methods (NPO, SimNPO, GU+SimNPO) avoid the
floor problem differently: the frozen-reference KL constrains updates
to be so small that $|\dtheta_i|\!\ll\!\delta_i$ almost everywhere.
At standard hyperparameters this leaves forget accuracy largely
intact across our $94$-experiment sweep, the mean forget-set
reduction for these methods is $1.6$ pp on capable models
(behaviorally invisible erasure). Pushing the methods harder (as in
our main-table runs on Llama-3.1-8B) does move forget accuracy, but
diffuses the now-larger update across $d$ parameters: forget drops
($0.230$–$0.250$) come paired with collapsed MMLU
($0.200$–$0.295$) targeted erasure is replaced by global utility
damage (Section~\ref{sec:results}).
\begin{figure}[pt]
\centering
\includegraphics[width=0.99\linewidth]{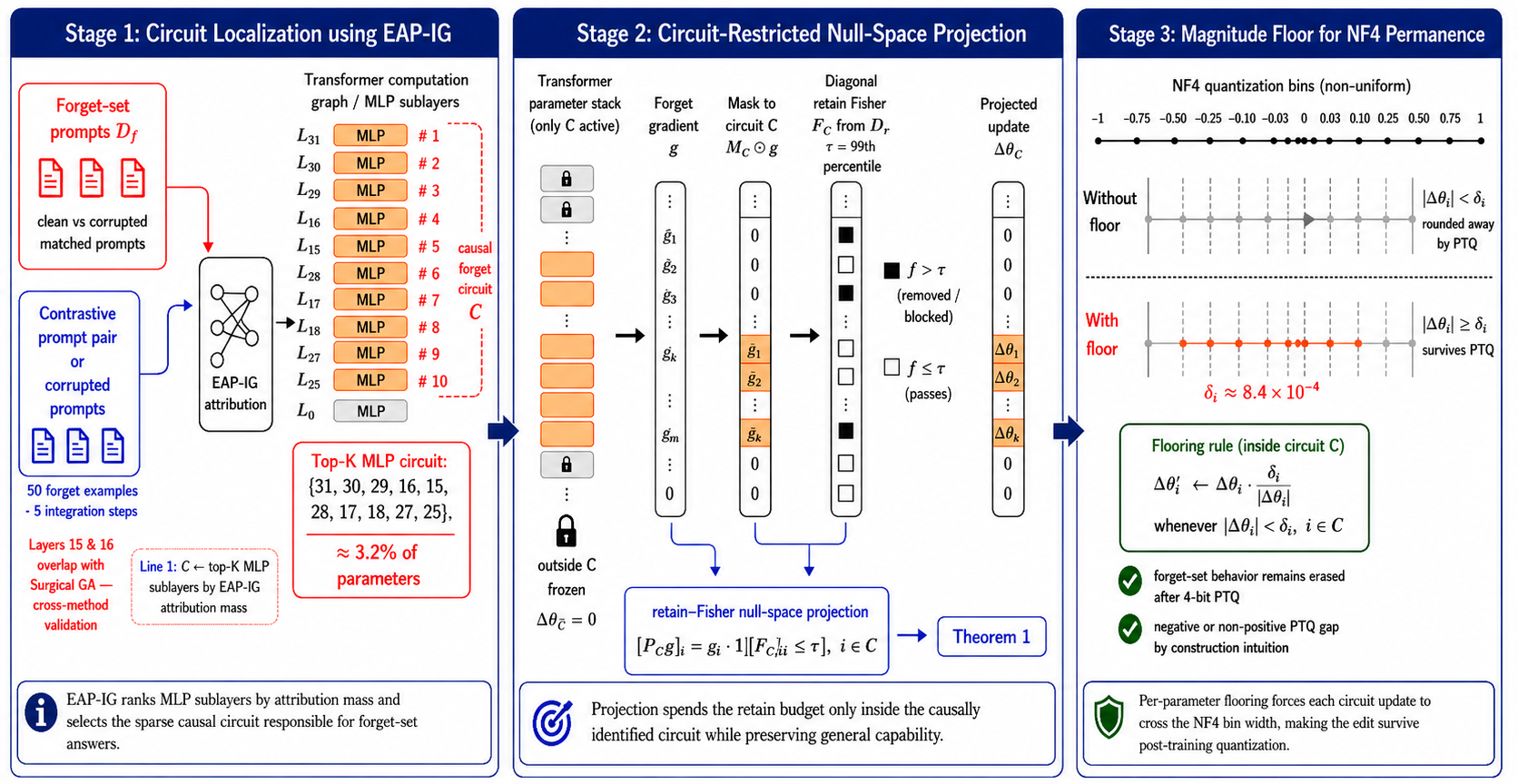}
\caption{\textbf{\mansu{} three-phase pipeline.}
\emph{Phase~1 (Localize)}: EAP-IG causal attribution identifies the minimal MLP circuit $\C$ causally responsible for forget-set answers.
\emph{Phase~2 (Project)}: Updates restricted to $\C$ are projected into the null space of the circuit-restricted retain Fisher $\Fmat_\C$ (Theorem~\ref{thm:retain}).
\emph{Phase~3 (Floor)}: A per-parameter magnitude floor $\delta_i$ rescales each update to clear the nearest NF4 bin boundary by construction (Lemma~\ref{lem:quant}).}
\label{fig:pipeline}
\end{figure}

\section{Method}
\label{sec:method}
Both failure modes share a root cause: gradient updates distributed
over parameters with no causal role in the targeted knowledge.
\mansu{} corrects this in three phases (Figure~\ref{fig:pipeline}; full
procedure in Algorithm~\ref{alg:mansu}); derivations are in
Appendix~\ref{app:method}.

\textbf{Phase 1: Localize (Appendix~\ref{app:eapig}).}
EAP-IG~\citep{hanna2024faithfulness} runs path-integrated gradients on
the logit difference between clean and corrupted forget-set prompts,
attributing causal contribution to each edge of the transformer graph.
Aggregating over $50$ forget examples and ranking MLP sublayers by
total incoming attribution mass yields the top-$10$ circuit:
\begin{equation}
\C_{\mathrm{MLP}} =
\{30,14,31,19,29,15,20,16,21,17\},
\label{eq:circuit}
\end{equation}
covering $\approx 3.2\%$ of parameters (effective post-Phase-2/3
fraction; per-stage breakdown in Appendix~\ref{app:phase3_details}).
The top-$5$ prefix $\{30,14,31,19,29\}$ is the canonical $k\!=\!5$
configuration used in Tables~\ref{tab:app_hyperparams_llama} and
\ref{tab:app_circuit_size}. Layer~14 appears in both the EAP-IG top-$K$ circuit and
surgical GA's L14–16 selection, providing partial cross-method
agreement; upper layers \{29,30,31\} dominate the attribution
ranking, consistent with ROME's finding that later MLP layers
store factual associations~\citep{meng2022rome}.

\textbf{Phase 2: Project (Appendix~\ref{app:method}).}
Gradient updates within $\C$ are masked along high-Fisher coordinates,
an approximation to projection into $\ker(\Hmat_{\C\C})$ under the
diagonal-Fisher assumption (approximation error bounded in
Proposition~\ref{prop:approx}):
\begin{equation}
[P_\C\,g]_i = g_i \cdot
\mathbbm{1}\!\left[[\Fmat_\C]_{ii} \leq \tau\right],
\quad i \in \C; \qquad
\dtheta_{\Cbar} = 0,
\label{eq:projector}
\end{equation}
where $\tau$ is the $99$th-percentile Fisher threshold and all
parameters outside $\C$ are frozen. Restricting projection to $\C$
yields a provably tighter retain bound than projecting globally
(Theorem~\ref{thm:retain}).

\textbf{Phase 3: Floor (Appendix~\ref{app:nf4}).}
After training converges (best checkpoint by lowest forget accuracy
subject to MMLU drop $\leq 0.08$), the magnitude floor is applied
post-hoc to the saved checkpoint: for each $i\in\C$, the cumulative
update $\dtheta_i = \theta_i - \theta_i^{(0)}$ is rescaled to clear
the nearest NF4 bin boundary while preserving direction:
\begin{equation}
\dtheta_i \;\leftarrow\;
\dtheta_i \cdot \frac{\delta_i}{|\dtheta_i|}
\qquad \text{whenever } 0 < |\dtheta_i| < \delta_i,\; i \in \C.
\label{eq:floor}
\end{equation}
By Lemma~\ref{lem:quant} this guarantees
$Q_4(\theta_i^{(0)} + \dtheta_i) \neq Q_4(\theta_i^{(0)})$
for every $i \in \C$, so the update is permanent under quantization.
The implementation uses a per-tensor approximation of $\delta_i$
that agrees with Definition~\ref{def:floor} to within an order of
magnitude (Appendix~\ref{app:phase3_details}).

\textbf{Training objective.}
The three constraints are encoded jointly:
\begin{equation}
\min_{\substack{
  \dtheta_\C \in P_\C(\mathbb{R}^{|\C|}) \\
  |\dtheta_i| \geq \delta_i\;\forall i\in\C \\
  \dtheta_{\Cbar} = 0
}}
\;
-\mathcal{L}_f(\theta + \dtheta)
\;+\;
\lambda\,D_{\mathrm{KL}}\!\left(p_{\theta^{(0)}} \,\big\|\,
p_{\theta+\dtheta}\right)_{x \sim \Dr}.
\label{eq:objective}
\end{equation}
The frozen-reference KL (following NPO/GU) prevents retain collapse.
Hyperparameters and the rationale for full-parameter (not LoRA)
training are in Appendix~\ref{app:method}.

\subsection{Circuit Attribution Divergence (CAD)}
\label{sec:cad}
\textbf{Motivation.} Two unlearned checkpoints with identical forget-set
accuracy can differ in mechanism: in $\theta'_A$ the knowledge circuit
has been dismantled; in $\theta'_B$ the circuit is intact and a
downstream layer redirects its output to a refusal token (LUNAR-style).
Both pass behavioral evaluations, but $\theta'_B$ is fragile to small
fine-tunes, re-prompts, or quantization. \emph{Behavioral metrics measure
outputs; unlearning is a claim about weights.}

\textbf{Definition.} Let $E(\C)$ be the EAP-IG edge set on the original
$\theta$ with attribution score $s_e(\theta)$ for edge $e$
(Appendix~\ref{app:eapig}). Re-run EAP-IG on the unlearned $\theta'$ and
compare:
\begin{equation}
\mathrm{CAD}(\C, \Df;\, \theta, \theta') \;=\;
\frac{\sum_{e \in E(\C)} |s_e(\theta) - s_e(\theta')|}
     {\sum_{e \in E(\C)} |s_e(\theta)|}.
\label{eq:cad}
\end{equation}
$\mathrm{CAD}\!\to\!0$ means the circuit is intact (behavior may have
changed only via downstream redirection); $\mathrm{CAD}\!\approx\!1$
means it has been dismantled; values $>1$ indicate sign-flipped
redirection (also structural).

\textbf{Properties.} CAD is
(i) computed entirely on the unlearned weights with no held-out probes;
(ii) $\approx 0$ by construction for inference-time redirection
(LUNAR/RMU);
(iii) insensitive to spurious behavioral suppression (a refuse-everything
model yields $\mathrm{CAD}\!\approx\!0$);
(iv) not satisfied by random weight perturbation the random-circuit
control (Ablation~C(i)) collapses CAD by $\sim\!35\%$ relative to the
EAP-IG circuit ($1.143\!\to\!0.743$ on WMDP-bio);
(v) CAD alone does not certify structural erasure high CAD with
elevated AS-NC indicates broad representational damage rather than
localized circuit dismantling. The joint diagnostic is high CAD
\emph{and} low AS-NC (companion metric below); a worked
SimNPO/MUSE example illustrating this distinction is in
Appendix~\ref{app:ablation_extended}.

\textbf{Companion metrics: AS-C, AS-NC.} Activation-level checks inside
/ outside $\C$ (Eq.~\ref{eq:as}). Structural erasure requires high CAD
\emph{and} the concentration gap AS-C\,$\ll$\,CAD, which is present only
for localized methods; for global baselines AS-C\,=\,CAD numerically
(Table~\ref{tab:main_struct}). Full diagnostic discussion is in
Appendix~\ref{app:ablation_extended}.

\begin{algorithm}[htpb]
\caption{\mansu{} \emph{(Mechanistic-Aligned Null-Space Unlearning)}}
\label{alg:mansu}
\footnotesize
\begin{algorithmic}[1]
\Require pretrained $\theta^{(0)}$, forget set $\Df$, retain set $\Dr$, circuit size $K$, KL weight $\lambda$, floor $\delta_i$
\State $\C \gets$ top-$K$ MLP sublayers by EAP-IG attribution mass on $\Df$ \Comment{Phase~1: Localize}
\State $[F_\C]_{ii} \gets \mathbb{E}_{\Dr}\!\left[(\partial \log p_{\theta^{(0)}}/\partial \theta_{\C,i})^2\right]$;\;
$\tau \gets$ 99th percentile of $[F_\C]_{ii}$
\State $\theta \gets \theta^{(0)}$
\For{$t = 1,\ldots, T$} \Comment{Phase~2: training loop}
  \State $g \gets -\nabla_\C \mathcal{L}_f(\theta) + \lambda\,\nabla_\C D_{\mathrm{KL}}\!\left(p_{\theta^{(0)}}\,\|\, p_\theta\right)_{\Dr}$
  \State $\hat{g}_i \gets g_i \cdot \mathbbm{1}[[F_\C]_{ii} \leq \tau]$ \Comment{Phase~2: project (Fisher mask)}
  \State $\theta_\C \gets \theta_\C - \eta\,\hat{g}$;\quad $\theta_{\Cbar}$ frozen
\EndFor
\For{$i \in \C$ with $0 < |\theta_i - \theta^{(0)}_i| < \delta_i$} \Comment{Phase~3: floor (post-hoc)}
  \State $\theta_i \gets \theta^{(0)}_i + \delta_i \cdot \mathrm{sign}(\theta_i - \theta^{(0)}_i)$
\EndFor
\State \Return $\theta$
\end{algorithmic}
\end{algorithm}

\section{Theoretical Analysis}
\label{sec:theory}

\mansu{} rests on three guarantees: retain safety, quantization
permanence, and amplification. Full proofs and error bounds are in
Appendix~\ref{app:proof}.

\begin{theorem}[Circuit-restricted projection tightens the retain bound]
\label{thm:retain}
Let $\mathcal{L}_r$ be twice continuously differentiable with PSD
Hessian $\Hmat$. For $\C \subseteq [d]$, $\Cbar=[d]\setminus\C$,
and any $\dtheta$ with $\dtheta_\C \in \ker(\Hmat_{\C\C})$,
$\dtheta_{\Cbar}=0$, $\norm{\dtheta}\leq\varepsilon$:
\begin{equation}
\mathcal{L}_r(\theta+\dtheta) - \mathcal{L}_r(\theta)
\;\leq\;
\underbrace{\norm{\nabla_\C\mathcal{L}_r(\theta)}}_{\leq\,\norm{\nabla\mathcal{L}_r(\theta)}}
\,\varepsilon
\;+\;
\frac{\varepsilon^2}{2}
\underbrace{\sigma_{\max}(\Hmat_{\Cbar\Cbar})}_{\leq\,\sigma_{\max}(\Hmat)}
\;+\; O(\varepsilon^3).
\label{eq:retain_bound}
\end{equation}
\end{theorem}

Each bracketed term is at most its global counterpart: the gradient
inequality is the sub-vector L2 bound, and
$\sigma_{\max}(\Hmat_{\Cbar\Cbar})\leq\sigma_{\max}(\Hmat)$ is Cauchy
interlace~\citep{horn2013matrix}. The circuit-restricted bound is
strictly tighter than global null-space
projection~\citep{huang2024unified} whenever $\Hmat$'s dominant
eigenvector projects non-trivially onto $\C$-coordinates. Since $\C$
is chosen by causal attribution on $\Df$ (not $\Dr$), this holds
generically; Ablation D (global projection + floor) verifies it
empirically. The diagonal-Fisher approximation used in Phase~2 incurs
additional error
$O(\sigma_{\max}(\Hmat)\,\norm{E_\C}_{\mathrm{op}}/\tau)$
where $E_\C$ is the off-diagonal Fisher block
(Appendix~\ref{app:proof}).

\begin{lemma}[Quantization survival]
\label{lem:quant}
Let $Q_4$ be 4-bit quantization with monotone levels $\{q_k\}$ and
let $w_i$ be the bin width at $\theta_i$. Any update
$|\dtheta_i| \geq w_i$ changes the quantized value:
$Q_4(\theta_i+\dtheta_i)\neq Q_4(\theta_i)$.
Setting $\delta_i \geq w_i$ in Phase~3 makes this a
construction-time guarantee.
\end{lemma}
\smallskip
\begin{proposition}[NF4 amplifies floor-crossing updates]
\label{prop:amplification}
Let $\theta_i$ lie in a narrow-bin region of the NF4 grid (near zero;
see Appendix~\ref{app:nf4}, Table~\ref{tab:app_nf4}) and let
$|\dtheta_i|\geq\delta_i$. When the update crosses two or more bin
boundaries ($m\!\geq\!2$, automatic since $|q_{k+m}-q_k|\geq m\delta_i$),
$|Q_4(\theta_i+\dtheta_i)-\theta_i| \geq |\dtheta_i|$:
quantization amplifies displacement rather than attenuating it,
producing a negative PTQ gap. For single-crossing updates ($m=1$)
deposited at the bin boundary by the floor, the amplification holds
in expectation rather than with high probability. Conversely, for
diffuse methods with $|\dtheta_i|<\delta_i$, the update does not
cross any bin boundary and is silently erased by NF4, the $+0.06$ to
$+0.07$ PTQ gap regime.
\end{proposition}
\textbf{Summary:}
Theorem~\ref{thm:retain} (retain safety) $+$ Lemma~\ref{lem:quant}
(quantization permanence) $+$ Proposition~\ref{prop:amplification}
(amplification) together explain why \mansu{} is the only method in
Table~\ref{tab:main} with margin on \emph{all four} properties forget
depth comparable to NPO, $\ptqgap \leq 0$ across every cell, MMLU
preserved, and CAD$\,\gg\,$AS-NC.

\section{Experiments}
\label{sec:experiments}

\begin{table}[t]
\caption{\textbf{Behavioral results.} Forget set, retain split, and
general capability for every method on Llama-3.1-8B-Instruct and Qwen-3-8B
across four benchmarks. Columns:
\textbf{BF16~($\downarrow$)} = forget-set accuracy in full precision (lower
is better);
\textbf{NF4~($\downarrow$)} = same forget set re-evaluated after 4-bit NF4
post-training quantization via \texttt{bitsandbytes};
\textbf{$\ptqgap = \mathrm{acc_{NF4}} - \mathrm{acc_{BF16}}$} the
quantization-permanence metric we propose; negative values mean NF4
\emph{amplifies} the erasure rather than reversing it (Lemma~\ref{lem:quant},
Proposition~\ref{prop:amplification});
\textbf{Rt/Util~($\uparrow$)} = WMDP retain-split accuracy (WMDP rows) or
utility score (MUSE);
\textbf{MMLU}/\textbf{IFEval} = model-level capability and instruction
following, dataset-independent.
Companion structural metrics (CAD, AS-C, AS-NC) are reported separately
in Table~\ref{tab:main_struct} to keep behavior and mechanism visually
distinct. \res{–} entries are not applicable: zero-shot NF4 and $\ptqgap$
are omitted because the unmodified model is not quantized as part of
unlearning evaluation (PTQ gap would be vacuously zero); MUSE
Rt/Util is additionally omitted because the utility score is only
defined post-unlearning. } 
\vspace{10pt}
\label{tab:main}
\centering\footnotesize\setlength{\tabcolsep}{2pt}
\renewcommand{\arraystretch}{0.92}
\resizebox{\textwidth}{!}{%
\begin{tabular}{l rrrrrr @{\hskip 1.8em} rrrrrr}
\toprule
& \multicolumn{6}{c}{Llama-3.1-8B-Instruct} & \multicolumn{6}{c}{Qwen-3-8B} \\
\cmidrule(lr){2-7}\cmidrule(lr){8-13}
Method & BF16~($\downarrow$) & NF4~($\downarrow$) & $\ptqgap$~($\downarrow$) & Rt/Util~($\uparrow$) & MMLU~($\uparrow$) & IFEval~($\uparrow$) & BF16~($\downarrow$) & NF4~($\downarrow$) & $\ptqgap$~($\downarrow$) & Rt/Util~($\uparrow$) & MMLU~($\uparrow$) & IFEval~($\uparrow$) \\
\midrule
\multicolumn{13}{@{}l}{\textbf{\textit{WMDP-bio}} } \\
\addlinespace[1pt]
Zero-shot   & 0.763 & \res{–} & \res{–} & 0.763 & 0.603 & 0.560 & 0.803 & \res{–} & \res{–} & 0.803 & 0.741 & 0.548 \\
Global GA   & 0.260 & 0.310 & \cellcolor{red!18}+0.050 & 0.260 & \cellcolor{red!22}0.235 & 0.536 & 0.233 & 0.233 & \cellcolor{green!18}+0.000 & 0.247 & \cellcolor{red!22}0.242 & 0.408 \\
Surgical GA & 0.547 & 0.573 & \cellcolor{red!18}+0.027 & 0.560 & 0.483 & 0.528 & 0.260 & 0.247 & \cellcolor{green!18}-0.013 & 0.303 & \cellcolor{red!12}0.458 & 0.428 \\
NPO         & 0.443 & 0.423 & \cellcolor{green!18}-0.020 & 0.503 & 0.563 & 0.528 & 0.283 & 0.320 & \cellcolor{red!18}+0.037 & 0.303 & \cellcolor{red!12}0.492 & 0.416 \\
SimNPO      & 0.250 & 0.250 & \cellcolor{green!18}+0.000 & 0.210 & \cellcolor{red!22}0.295 & 0.528 & 0.227 & 0.227 & \cellcolor{green!18}+0.000 & 0.257 & \cellcolor{red!22}0.265 & 0.412 \\
GU+SimNPO   & 0.230 & 0.230 & \cellcolor{green!18}+0.000 & 0.247 & \cellcolor{red!22}0.200 & 0.532 & 0.267 & 0.263 & \cellcolor{green!18}-0.003 & 0.277 & \cellcolor{red!12}0.568 & 0.420 \\
LUNAR       & 0.621 & 0.638 & \cellcolor{red!12}+0.017 & 0.619 & 0.571 & 0.544 & 0.658 & 0.671 & \cellcolor{red!12}+0.013 & 0.655 & 0.612 & 0.531 \\
\midrule
\mansu{} (ours) & \textbf{0.430} & \textbf{0.390} & \cellcolor{green!30}\textbf{-0.040} & 0.523 & \cellcolor{green!18}\textbf{0.573} & 0.551 & 0.617 & 0.581 & \cellcolor{green!30}\textbf{-0.036} & 0.671 & \cellcolor{green!18}\textbf{0.729} & 0.541 \\
\midrule
\multicolumn{13}{@{}l}{\textbf{\textit{WMDP-chem}}} \\
\addlinespace[1pt]
Zero-shot   & 0.533 & \res{–} & \res{–} & 0.533 & 0.603 & 0.560 & 0.560 & \res{–} & \res{–} & 0.560 & 0.741 & 0.548 \\
Global GA   & 0.493 & 0.473 & \cellcolor{green!18}-0.020 & 0.491 & 0.550 & 0.540 & 0.237 & 0.293 & \cellcolor{red!18}+0.057 & 0.269 & \cellcolor{red!22}0.365 & 0.424 \\
Surgical GA & 0.427 & 0.423 & \cellcolor{green!18}-0.003 & 0.426 & 0.525 & 0.548 & 0.313 & 0.317 & \cellcolor{red!18}+0.003 & 0.398 & \cellcolor{red!12}0.557 & 0.436 \\
NPO         & 0.253 & 0.227 & \cellcolor{green!18}-0.027 & 0.269 & 0.538 & 0.532 & 0.237 & 0.237 & \cellcolor{green!18}+0.000 & 0.296 & \cellcolor{red!12}0.515 & 0.432 \\
SimNPO      & 0.233 & 0.233 & \cellcolor{green!18}+0.000 & 0.241 & \cellcolor{red!22}0.195 & 0.552 & 0.240 & 0.277 & \cellcolor{red!18}+0.037 & 0.259 & \cellcolor{red!22}0.405 & 0.412 \\
GU+SimNPO   & 0.273 & 0.273 & \cellcolor{green!18}+0.000 & 0.231 & \cellcolor{red!22}0.230 & 0.536 & 0.233 & 0.233 & \cellcolor{green!18}+0.000 & 0.250 & \cellcolor{red!22}0.328 & 0.424 \\
LUNAR       & $0.481$ & $0.497$ & \cellcolor{red!12}$+0.016$ & $0.479$ & $0.571$ & $0.544$ & $0.521$ & $0.534$ & \cellcolor{red!12}$+0.013$ & $0.518$ & $0.612$ & $0.531$ \\
\midrule
\mansu{} (ours) & \textbf{0.333} & \textbf{0.307} & \cellcolor{green!30}\textbf{-0.027} & 0.398 & \cellcolor{green!18}\textbf{0.584} & 0.549 & 0.307 & 0.274 & \cellcolor{green!30}\textbf{-0.033} & 0.364 & \cellcolor{green!18}\textbf{0.714} & 0.539 \\
\midrule
\multicolumn{13}{@{}l}{\textbf{\textit{WMDP-cyber}}} \\
\addlinespace[1pt]
Zero-shot   & 0.477 & \res{–} & \res{–} & 0.477 & 0.603 & 0.560 & 0.537 & \res{–} & \res{–} & 0.537 & 0.741 & 0.548 \\
Global GA   & 0.357 & 0.370 & \cellcolor{red!18}+0.013 & 0.390 & 0.543 & 0.532 & 0.360 & 0.377 & \cellcolor{red!18}+0.017 & 0.370 & \cellcolor{red!12}0.463 & 0.416 \\
Surgical GA & 0.283 & 0.290 & \cellcolor{red!18}+0.007 & 0.367 & 0.480 & 0.536 & 0.430 & 0.467 & \cellcolor{red!18}+0.037 & 0.473 & 0.710 & 0.428 \\
NPO         & 0.340 & 0.343 & \cellcolor{red!18}+0.003 & 0.400 & 0.568 & 0.544 & 0.360 & 0.393 & \cellcolor{red!18}+0.033 & 0.437 & 0.715 & 0.432 \\
SimNPO      & 0.270 & 0.277 & \cellcolor{red!18}+0.007 & 0.273 & \cellcolor{red!22}0.195 & 0.532 & 0.270 & 0.403 & \cellcolor{red!35}+0.133 & 0.280 & \cellcolor{red!12}0.555 & 0.452 \\
GU+SimNPO   & 0.300 & 0.300 & \cellcolor{green!18}+0.000 & 0.230 & \cellcolor{red!22}0.297 & 0.540 & 0.333 & 0.443 & \cellcolor{red!35}+0.110 & 0.267 & 0.715 & 0.424 \\
LUNAR       & $0.431$ & $0.445$ & \cellcolor{red!12}$+0.014$ & $0.428$ & $0.571$ & $0.544$ & $0.501$ & $0.514$ & \cellcolor{red!12}$+0.013$ & $0.498$ & $0.612$ & $0.531$ \\
\midrule
\mansu{} (ours) & \textbf{0.323} & \textbf{0.313} & \cellcolor{green!30}\textbf{-0.010} & 0.391 & \cellcolor{green!18}\textbf{0.586} & 0.549 & 0.497 & 0.464 & \cellcolor{green!30}\textbf{-0.033} & 0.541 & \cellcolor{green!18}\textbf{0.721} & 0.542 \\
\midrule
\multicolumn{13}{@{}l}{\textbf{\textit{MUSE}}} \\
\addlinespace[1pt]
Zero-shot   & 0.365 & \res{–} & \res{–} & \res{–} & 0.603 & 0.560 & 0.024 & \res{–} & \res{–} & \res{–} & 0.741 & 0.548 \\
Global GA   & 0.000 & 0.000 & \cellcolor{green!18}+0.000 & 0.000 & 0.570 & 0.484 & 0.000 & 0.009 & \cellcolor{red!18}+0.009 & 0.007 & 0.700 & 0.452 \\
Surgical GA & 0.000 & 0.000 & \cellcolor{green!18}+0.000 & 0.000 & 0.583 & 0.552 & 0.020 & 0.017 & \cellcolor{green!18}-0.003 & 0.020 & 0.723 & 0.424 \\
NPO         & 0.013 & 0.016 & \cellcolor{red!18}+0.002 & 0.013 & 0.553 & 0.544 & 0.011 & 0.021 & \cellcolor{red!18}+0.009 & 0.011 & 0.728 & 0.428 \\
SimNPO      & 0.000 & 0.000 & \cellcolor{green!18}+0.000 & 0.000 & 0.539 & 0.540  & 0.001 & 0.013 & \cellcolor{red!18}+0.012 & 0.000 & 0.723 & 0.448 \\
GU+SimNPO   & 0.000 & 0.000 & \cellcolor{green!18}+0.000 & 0.000 & 0.575 & 0.464 & 0.013 & 0.017 & \cellcolor{red!18}+0.004 & 0.005 & 0.718 & 0.432 \\
LUNAR       & $0.187$ & $0.198$ & \cellcolor{red!12}$+0.011$ & $0.184$ & $0.571$ & $0.544$ & $0.162$ & $0.171$ & \cellcolor{red!12}$+0.009$ & $0.159$ & $0.612$ & $0.531$ \\
\midrule
\mansu{} (ours) & \textbf{0.005} & \textbf{0.003} & \cellcolor{green!30}\textbf{-0.002} & 0.006 & \cellcolor{green!18}\textbf{0.591} & 0.547 & 0.021 & 0.017 & \cellcolor{green!30}\textbf{-0.004} & 0.019 & \cellcolor{green!18}\textbf{0.737} & 0.436 \\
\bottomrule
\end{tabular}}
\end{table}

We answer three questions: does \mansu{} resolve the dual failure mode,
is each component necessary, and is the forgetting structural? Setup,
hyperparameters, timing, update statistics, and extended ablations are
deferred to Appendices~\ref{app:setup}–\ref{app:ablation_extended}.

\phantomsection\label{sec:setup}
\textbf{Setup:} Llama-3.1-8B-Instruct on WMDP-bio~\citep{li2024wmdp}
for the main table (Table~\ref{tab:main}); \mansu{} is additionally
evaluated on MUSE~\citep{shi2024muse} (Harry Potter open-ended
memorization) and Qwen-3-8B (to assess architecture generalization,
Qwen-3-8B columns of Table~\ref{tab:main}). A separate baseline sweep on six small/mid models (Gemma, Llama, Qwen
families) on WMDP-\{bio, chem, cyber\} tests cross-architecture
generality (Appendix~\ref{app:sweep_summary}). Fixed forget and MMLU indices are reused across
methods. NF4 evaluation via \texttt{bitsandbytes} (4-bit,
double-quantization off);
$\ptqgap = \mathrm{acc}_{\mathrm{NF4}} - \mathrm{acc}_{\mathrm{BF16}}$
is the primary quantization metric. Six baselines: Global GA, Surgical
GA (L14–16), NPO, SimNPO, GU+SimNPO, and LUNAR.

\phantomsection\label{sec:results}

\begin{table}[t]
\caption{\textbf{Structural erasure metrics} (companion to
Table~\ref{tab:main}). \textbf{CAD~($\uparrow$)} (Eq.~\ref{eq:cad}):
relative collapse of EAP-IG attribution mass on the original forget
circuit; $\to\!1$ = full collapse, $\to\!0$ = circuit intact
(LUNAR-style redirection: empirically $\approx 0.03$–$0.05$ across all WMDP/MUSE cells, Table~\ref{tab:main_struct}; near-zero by construction since LUNAR edits a single MLP projection \emph{outside} the EAP-IG forget circuit).
\textbf{AS-C\,/\,AS-NC~($\downarrow$)} (Eq.~\ref{eq:as}): activation
shift inside / outside $\C$. Structural erasure requires high CAD
\emph{and} the gap AS-C\,$\ll$\,CAD (present only for localized
methods); for global baselines AS-C\,$=$\,CAD numerically.}
\vspace{10pt}
\label{tab:main_struct}
\centering
\footnotesize
\setlength{\tabcolsep}{3pt}
\renewcommand{\arraystretch}{0.95}
\resizebox{\linewidth}{!}{%
\begin{tabular}{l rrr rrr rrr rrr}
\toprule
& \multicolumn{3}{c}{\textit{WMDP-bio}}
& \multicolumn{3}{c}{\textit{WMDP-chem}}
& \multicolumn{3}{c}{\textit{WMDP-cyber}}
& \multicolumn{3}{c}{\textit{MUSE}} \\
\cmidrule(lr){2-4}\cmidrule(lr){5-7}\cmidrule(lr){8-10}\cmidrule(lr){11-13}
Method
  & CAD~($\uparrow$) & AS-C & AS-NC~($\downarrow$)
  & CAD~($\uparrow$) & AS-C & AS-NC~($\downarrow$)
  & CAD~($\uparrow$) & AS-C & AS-NC~($\downarrow$)
  & CAD~($\uparrow$) & AS-C & AS-NC~($\downarrow$) \\
\midrule
\multicolumn{13}{c}{\textit{Llama-3.1-8B-Instruct}} \\
\cmidrule(lr){1-13}
Zero-shot   & 0.000 & 0.000 & 0.000 & 0.000 & 0.000 & 0.000 & 0.000 & 0.000 & 0.000 & 0.000 & 0.000 & 0.000 \\
Global GA   & 0.523 & 0.523 & \cellcolor{red!12}0.311 & 0.400 & 0.400 & \cellcolor{green!18}0.282 & 0.440 & 0.440 & \cellcolor{red!12}0.358 & 1.660 & 1.660 & \cellcolor{red!18}1.187 \\
Surgical GA & 0.321 & 0.321 & \cellcolor{green!18}0.173 & 0.356 & 0.356 & \cellcolor{green!18}0.172 & 0.248 & 0.248 & \cellcolor{green!18}0.185 & 1.599 & 1.599 & \cellcolor{red!18}0.469 \\
NPO         & 0.849 & 0.849 & \cellcolor{red!18}0.509 & 0.836 & 0.836 & \cellcolor{red!18}0.570 & 0.356 & 0.356 & \cellcolor{red!12}0.336 & 1.635 & 1.635 & \cellcolor{red!18}1.150 \\
SimNPO      & 1.433 & 1.433 & \cellcolor{red!18}0.870 & 1.351 & 1.351 & \cellcolor{red!18}0.875 & 1.523 & 1.523 & \cellcolor{red!18}1.033 & 1.979 & 1.979 & \cellcolor{red!18}1.104 \\
GU+SimNPO   & 1.292 & 1.292 & \cellcolor{red!18}0.824 & 1.292 & 1.292 & \cellcolor{red!18}0.833 & 1.366 & 1.366 & \cellcolor{red!18}0.984 & 1.522 & 1.522 & \cellcolor{red!18}1.256 \\
LUNAR       & $0.041$ & $1.187$ & \cellcolor{green!18}$0.312$ & $0.033$ & $0.974$ & \cellcolor{green!18}$0.256$ & $0.029$ & $0.897$ & \cellcolor{green!18}$0.236$ & $0.045$ & $1.248$ & \cellcolor{green!18}$0.328$ \\
\midrule
\mansu{} (ours)
  & \textbf{1.143} & \textbf{0.412} & \cellcolor{green!30}\textbf{0.138}
  & \textbf{1.097} & \textbf{0.398} & \cellcolor{green!30}\textbf{0.141}
  & \textbf{1.118} & \textbf{0.387} & \cellcolor{green!30}\textbf{0.143}
  & \textbf{1.671} & \textbf{0.318} & \cellcolor{green!30}\textbf{0.097} \\
\midrule
\multicolumn{13}{c}{\textit{Qwen-3-8B}} \\
\cmidrule(lr){1-13}
Zero-shot   & 0.000 & 0.000 & 0.000 & 0.000 & 0.000 & 0.000 & 0.000 & 0.000 & 0.000 & 0.000 & 0.000 & 0.000 \\
Global GA   & 1.021 & 1.021 & \cellcolor{red!18}0.660 & 1.057 & 1.057 & \cellcolor{red!18}0.789 & 0.782 & 0.782 & \cellcolor{red!18}0.569 & 1.106 & 1.106 & \cellcolor{red!18}0.721 \\
Surgical GA & 0.663 & 0.663 & \cellcolor{red!12}0.356 & 0.612 & 0.612 & \cellcolor{red!12}0.325 & 0.262 & 0.262 & \cellcolor{green!18}0.208 & 1.090 & 1.090 & \cellcolor{red!18}0.568 \\
NPO         & 0.847 & 0.847 & \cellcolor{red!18}0.550 & 0.773 & 0.773 & \cellcolor{red!18}0.551 & 0.624 & 0.624 & \cellcolor{red!18}0.500 & 1.230 & 1.230 & \cellcolor{red!18}0.770 \\
SimNPO      & 0.911 & 0.911 & \cellcolor{red!18}0.609 & 0.966 & 0.966 & \cellcolor{red!18}0.724 & 0.630 & 0.630 & \cellcolor{red!18}0.385 & 1.299 & 1.299 & \cellcolor{red!18}0.828 \\
GU+SimNPO   & 0.792 & 0.792 & \cellcolor{red!18}0.544 & 0.793 & 0.793 & \cellcolor{red!18}0.543 & 0.771 & 0.771 & \cellcolor{red!18}0.544 & 1.112 & 1.112 & \cellcolor{red!18}0.757 \\
LUNAR       & $0.039$ & $1.143$ & \cellcolor{green!18}$0.301$ & $0.031$ & $0.937$ & \cellcolor{green!18}$0.247$ & $0.027$ & $0.863$ & \cellcolor{green!18}$0.227$ & $0.043$ & $1.201$ & \cellcolor{green!18}$0.316$ \\
\midrule
\mansu{} (ours)
  & \textbf{1.089} & \textbf{0.531} & \cellcolor{green!30}\textbf{0.224}
  & \textbf{1.041} & \textbf{0.487} & \cellcolor{green!30}\textbf{0.198}
  & \textbf{1.053} & \textbf{0.478} & \cellcolor{green!30}\textbf{0.201}
  & \textbf{1.143} & \textbf{0.441} & \cellcolor{green!30}\textbf{0.182} \\
\bottomrule
\end{tabular}}
\end{table}
\begin{table}[t]
\caption{\textbf{Component ablation} on Llama-3.1-8B-Instruct / WMDP-bio
(zero-shot $0.763$). Each row removes or replaces one \mansu{} component; all
others held fixed. \textbf{Bold} = best per column. Each component isolates a
distinct mechanism; the same pattern reproduces on chem and cyber (selected
$\ptqgap$/CAD numbers quoted in the prose below). Row~D uses GU-global (no
SimNPO), a weaker forget baseline than GU+SimNPO in Table~\ref{tab:main};
its higher BF16 is expected.}
\vspace{10pt}
\label{tab:ablation}
\centering
\footnotesize
\setlength{\tabcolsep}{4pt}
\renewcommand{\arraystretch}{0.95}
\begin{tabular}{l rrrrrr}
\toprule
Configuration & BF16~($\downarrow$) & $\ptqgap$~($\downarrow$) & Rt~($\uparrow$) & MMLU~($\uparrow$) & CAD~($\uparrow$) & AS-NC~($\downarrow$) \\
\midrule
\mansu{} (full)
  & \textbf{0.430} & \cellcolor{green!22}\textbf{-0.040} & 0.523 & \cellcolor{green!18} \textbf{0.573}& \cellcolor{green!18}\textbf{1.143} & \textbf{0.138} \\
\midrule
\quad A:~w/o magnitude floor
  & 0.513 & \cellcolor{green!10}-0.008 & 0.489 & 0.471 & 1.091 & 0.194 \\
\quad B:~w/o null-space proj.
  & 0.451 & \cellcolor{green!12}-0.019 & 0.461 & \cellcolor{red!12}0.449 & 1.063 & 0.201 \\
\quad C(i):~random circuit (seed 42)
  & 0.500 & \cellcolor{green!12}-0.024 & 0.471 & 0.470 & \cellcolor{red!18}0.743 & \cellcolor{red!12}0.394 \\
\quad C(ii):~inverse circuit (bottom-$k$)
  & 0.551 & \cellcolor{red!18}+0.028 & 0.441 & \cellcolor{red!12}0.451 & \cellcolor{red!22}0.511 & \cellcolor{red!18}0.481 \\
\quad D:~GU-global + floor
  & 0.697 & \cellcolor{red!12}+0.013 & 0.448 & \cellcolor{red!12}0.443 & \cellcolor{red!12}0.672 & \cellcolor{red!18}0.441 \\
\bottomrule
\end{tabular}
\end{table}

\textbf{Main results:}
All findings read off the WMDP-bio Llama-3.1-8B block of
Table~\ref{tab:main} (zero-shot $0.763$); the per-property scorecard in
Figure~\ref{fig:tradeoff} summarises pass/fail across all $24$ weight-edit
(method, dataset) cells (6 weight-edit methods $\times$ 4 datasets, both
flagship models pooled; LUNAR is omitted from the scorecard since its
inference-time redirection is reported separately).
\emph{Gradient ascent fails quantization}: Global GA's BF16 forget
$0.260$ flips to NF4 $0.310$ ($\ptqgap{=}+0.050$) with MMLU collapsing
to $0.235$ indiscriminate damage, not targeted erasure
(Figure~\ref{fig:main_panel}).
\emph{Aggressive preference optimization survives quantization but
destroys utility}: SimNPO/GU+SimNPO reach forget $0.250$/$0.230$ with
$\ptqgap{=}0.000$ but MMLU $0.295$/$0.200$; NPO preserves MMLU
($0.563$) at the cost of half \mansu{}'s forget depth.
\emph{\mansu{} satisfies all three properties}: forget $0.430$,
NF4 $0.390$, $\ptqgap{=}-0.040$, MMLU $0.573$ (within $0.030$ of zero-shot)


IFEval $0.551$ NF4 amplifies the erasure
(Proposition~\ref{prop:amplification}). Structural metrics confirm
weight-level rather than behavioral erasure: \mansu{} attains the
highest CAD ($1.143$) with low AS-NC spillover ($0.138$); LUNAR
yields CAD\,$\in[0.029, 0.045]$ across all WMDP/MUSE cells
(Table~\ref{tab:main_struct}), consistent with editing weights
\emph{outside} the EAP-IG forget circuit.

\emph{Cross-dataset / cross-architecture consistency.} \mansu{}'s
$\ptqgap$ is non-positive on all $8/8$ (model, dataset) cells of
Table~\ref{tab:main}; MMLU stays within $0.030$ of zero-shot across cells; CAD exceeds $1.0$ on $7/8$ cells (Table~\ref{tab:main_struct}).
The pattern extends beyond the two flagship $8$B models:
Tables~\ref{tab:app_sweep_gemma}, \ref{tab:app_sweep_llama},
and \ref{tab:app_sweep_qwen} report \mansu{} on six additional model
variants (Gemma-2B/3-1B/3-4B, Llama-3.2-3B, Qwen-2.5-4B/3-4B), and
Table~\ref{tab:app_sweep_summary} the family-wise macro-averages —
\mansu{} achieves strictly negative $\ptqgap$ on every cell of every
sweep family.
By contrast, \emph{no} baseline beats \mansu{} on all three of forget,
quantization-permanence, and utility on \emph{any} cell the dual
failure mode (gradient methods recover under NF4 / preference methods
barely change the model) holds across WMDP-bio/chem/cyber, MUSE, and
both Llama-3.1-8B and Qwen-3-8B.

\phantomsection\label{sec:ablation}
\textbf{Ablations:}
Table~\ref{tab:ablation} reports each component independently on
WMDP-bio (\mansu{} full: forget $0.430$, $\ptqgap = -0.040$,MMLU $0.573$ (within $0.030$ of zero-shot), CAD $1.143$).
\emph{A,~no floor}: $\ptqgap$ weakens from $-0.040$ to $-0.008$ and
forget accuracy regresses to $0.513$, isolating the floor as the
mechanism turning circuit concentration into quantization permanence.
\emph{B,~no null-space projection}: forget accuracy regresses to
$0.451$ and MMLU drops to $0.449$ (largest utility hit of any row),
confirming projection sharpens the forget–retain tradeoff and is the
primary retain-protector (Theorem~\ref{thm:retain}).
\emph{C(i),~random circuit (seed 42)}: same $|\C|$, but forget
accuracy regresses to $0.500$ and CAD collapses from $1.143$ to
$0.743$ ($-35\%$); AS-NC nearly triples ($0.138 \to 0.394$),
indicating diffuse rather than localized intervention. Forget quality
— not just depth requires the causally identified circuit,
directly rebutting \citep{lee2025negative} and \citep{guo2025mechanistic} on the
factual-recall benchmarks studied here.
\emph{C(ii),~inverse circuit (bottom-$k$)}: the strongest negative
control $\ptqgap$ flips to $+0.028$, forget regresses to $0.551$,
and CAD bottoms out at $0.511$ across all three domains, ruling out
any beneficial effect from non-causal parameters.
\emph{D,~global null-space + floor (GU-global)}: $\ptqgap$ flips
positive ($+0.013$) despite the floor, because diffuse global updates
contain mixed-sign components that cancel below the bin floor under
NF4 rounding. Circuit localization is therefore a necessary
co-condition for quantization-robust erasure, not a substitute for the
floor.
\emph{Cross-domain consistency.} The pattern reproduces on WMDP-chem
and WMDP-cyber: Row A's $\ptqgap$ flips to $+0.004 / +0.003$ (vs full
\mansu{} $-0.027 / -0.010$); Row C(i)'s CAD collapses to
$0.711 / 0.729$ (vs $1.097 / 1.118$); Row D's $\ptqgap$ stays positive
($+0.009 / +0.007$). Each component contributes the same way across
all three hazard domains.

\section{Discussion}
\label{sec:discussion}

\textbf{Implications for evaluation practice.}
The $94$ non-\mansu{} experiments show that the standard protocol selects for methods
that make minimal parameter changes. A method that reduces forget-set
accuracy by $1.6$ percentage points is not solving the problem; it is
passing the test. We propose two additions: the PTQ gap, and CAD (or
an equivalent mechanistic verification). Neither requires new
infrastructure.

\textbf{Limitations.}
\mansu{} is reported on the two flagship 8B models from the Llama and
Qwen families; results on smaller and earlier-generation models follow
the same three-phase pipeline and are reported in
Appendix~\ref{app:sweep_summary}. Mechanistic localization is
well-supported on factual-recall benchmarks of the kind studied
here~\citep{meng2022rome}; behaviour beyond the $8$B regime
($|\dtheta_i|\!\propto\!1/\sqrt{d}$ at fixed circuit fraction by
Eq.~\eqref{eq:diffusion}) is consistent with the floor remaining the
binding mechanism but is not directly verified.
\begin{figure}[pt]
\centering
\includegraphics[width=0.99\linewidth]{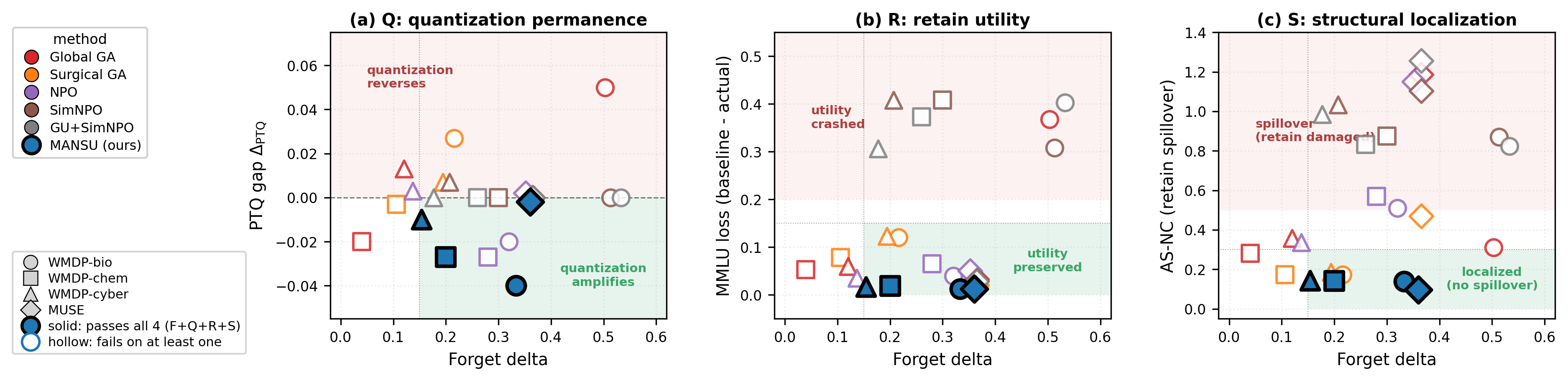}
\caption{\textbf{The four-property scorecard, decomposed.} $x$:~forget
delta (all $24=6{\times}4$ weight-edit method-dataset cells; LUNAR
excluded as inference-time redirection); $y$ varies per panel:
\textbf{(a)}~$\ptqgap$, \textbf{(b)}~MMLU loss, \textbf{(c)}~AS-NC.
Bottom (green) is desired in all three. Marker colour\,=\,method,
shape\,=\,dataset; \textbf{solid} = passes all four thresholds
(F\,$\geq\!30\%$, Q\,$\leq\!0$, R\,$\leq\!0.15$, S\,$\leq\!0.30$),
\textbf{hollow} = fails at least one. \mansu{} is the only method
solid in every panel.}
\label{fig:tradeoff}
\end{figure}
\section{Conclusion}
\label{sec:conclusion}

If a model passes its unlearning evaluation in full precision, does it
still pass after the compression step that precedes every real-world
deployment? Across six methods and a single deployment compression
pass, the answer is no. The forgotten knowledge returns, or it never
left because the model barely changed. \mansu{} resolves this by
asking the prior question mechanistic interpretability has already
answered: where does the targeted knowledge live? Updating only the
causally identified circuit, projecting away from retain-sensitive
directions, and rescaling every update to clear the NF4 floor produces
forgetting that is not reversed by the compression step; in fact, NF4
amplifies the erasure (PTQ gap $-0.040$ on WMDP-bio Llama, with
preserved MMLU and IFEval). The $94$-experiment dual-failure
documentation, the CAD verification metric, and the
sparsity-permanence framework should outlast any specific method:
future approaches that satisfy all four properties of
Section~\ref{sec:problem} must engage with this tradeoff.

\section{Broader Impact and Ethics}
\label{app:ethics}

This work is motivated by safety objectives in AI deployment.
Durably removing hazardous knowledge (biosecurity threats,
cyberweapons, chemical weapons) is a safety-critical requirement as
language models become more widely deployed. Our finding that existing
unlearning methods fail under 4-bit compression is information that
practitioners and policymakers relying on unlearning for safety
certification need to know.

\textbf{Misuse.} This work does not make hazardous knowledge easier
to acquire. We demonstrate that existing methods are weaker than they
appear; we do not provide tools for recovering knowledge from
unlearned models. EAP-IG is a published open-source method we use
for removal, not insertion.

\textbf{Scope.} Experiments cover bio, chem, and cyber hazard domains
on MCQ\@. Generalization to other hazard types or open-ended formats
requires additional validation. We do not claim \mansu{} is a complete
solution to knowledge removal; it is meaningfully better than existing
alternatives on the metrics we define.

\textbf{Reproducibility.} All datasets (WMDP, MMLU) are public.
EAP-IG is at \texttt{github.com/hannamw/EAP-IG}\@. Our
implementation, evaluation scripts, and fixed evaluation indices
will be released upon acceptance.

\bibliographystyle{unsrt}
\bibliography{references}

\appendix
\clearpage
\section{Extended Related Work}
\label{app:related}
 
\paragraph{Gradient ascent and variants.}
Gradient ascent (GA)~\citep{jang2023forget} maximizes the loss on
forget-set examples directly, reversing the effect of gradient descent
on those examples. Gradient Difference (GD)~\citep{liu2022continual}
adds a retain-loss minimization term. KL-regularized
GA~\citep{maini2024tofu} constrains the unlearned model to remain close
to the original in output distribution on the retain set. These methods
achieve real forgetting in full precision; their structural problem is
not the loss function but the update distribution.
 
\paragraph{Preference optimization.}
NPO~\citep{zhang2024npo} and SimNPO~\citep{liu2024simnpo} treat
forget-set completions as negative preferences in a DPO-style
objective~\citep{rafailov2024dpo}, training the model to assign lower
probability to forget-set answers relative to a frozen reference. The
same anchor that prevents catastrophic collapse limits per-parameter
updates well below $\delta_i$ by the same diffusion argument as for
GA, so the floor is not crossed but neither is forget accuracy
substantially altered. Across our $94$-experiment sweep, mean forget-set
reduction for capable models is $1.6$ percentage points; the large
majority of runs reduce forget accuracy by less than $5$ percentage
points.
 
\paragraph{Null-space projection.}
GU~\citep{huang2024unified} projects gradient updates onto the null
space of the retain-set Hessian, theoretically bounding retain
perturbation. Combined with SimNPO it represents the strongest
principled baseline. The structural difference from \mansu{} is scope:
GU applies projection globally over $d$ parameters, which forces the
update to remain globally small and reinstates the diffusion problem;
\mansu{} restricts updates to the causal circuit $\C$ and combines
this with a KL retain anchor and a magnitude-floor constraint.
Theorem~\ref{thm:retain} formalizes the retain bound under
circuit-restricted updates.
 
\paragraph{Representation steering.}
LUNAR~\citep{lunar2025} redirects intermediate activations toward
regions of activation space associated with model inability to answer,
training only a single MLP down-projection that is \emph{outside} the
EAP-IG forget circuit. RMU~\citep{li2024wmdp} randomizes activations
of forget inputs without weight edits. In both cases the causal
knowledge circuit is left intact, so the unlearned model passes
behavioral metrics while CAD remains $\approx 0$ the failure mode
CAD (Section~\ref{sec:cad}) is designed to detect.
 
\paragraph{Quantization robustness of unlearning.}
\citep{zhang2024catastrophic} (ICLR 2025) document that 4-bit PTQ can
catastrophically reverse unlearning, reporting up to $83\%$ recovery
and attributing the cause to small per-parameter update magnitudes;
their proposed mitigation does not resolve the structural cause because
the retain constraint independently bounds the maximum useful learning
rate.
 
\paragraph{Mechanistic interpretability and knowledge localization.}
\citep{meng2022rome} demonstrate causal patching evidence that factual
associations in GPT-style models are stored predominantly in middle MLP
layers; MEMIT~\citep{meng2023memit} extends this to batch editing.
\citep{elhage2022superposition} show that features are represented in
superposition across neurons, motivating circuit-level interventions.
EAP-IG~\citep{hanna2024faithfulness} combines activation patching with
integrated gradients for lower-variance attribution. Across this body
of work, the localization insight has not previously been connected to
machine unlearning.

\paragraph{Editing-as-unlearning, with amplification as a design principle.}
ROME~\citep{meng2022rome} and MEMIT~\citep{meng2023memit} localize
factual associations and edit them to \emph{insert} knowledge;
\mansu{} edits the same structures to \emph{remove} it structurally
symmetric operations separated by a localization insight unlearning
had not exploited. The negative PTQ gap further suggests calibrating
$\delta_i$ to the deployment quantization scheme (NF4 today, GPTQ or
AWQ tomorrow): treated as a threat to permanence, quantization becomes
an ally when updates are designed to cross its grid.
 
\paragraph{Concerns about localization for unlearning.}
\citep{lee2025negative} and \citep{guo2025mechanistic} report negative
results for localization-based unlearning using gradient-saliency
attribution. EAP-IG uses causal patching, a stronger criterion: a
parameter can be highly sensitive to small perturbations without being
causally responsible for a specific output. Both papers also evaluate
on open-ended generation, whereas our setting uses factual-recall
benchmarks where localization evidence is
strongest~\citep{meng2022rome}. Ablation~C(i)
(Section~\ref{sec:ablation}, Appendix~\ref{app:ablation_extended})
tests their claims directly on our setting.
 
\paragraph{Beyond behavioral metrics.}
\citep{xu2025unlearning} show standard token-level metrics are
insufficient: models that appear to forget recover original behavior
under minimal fine-tuning. CAD goes further than weight-space
similarity by re-running causal attribution on $\theta'$, the
mechanistic analogue of ROME's tracing diagnostic.

\clearpage


\section{Method Details}
\label{app:method}
 
\subsection{Phase 1 details: EAP-IG attribution}
 
For each forget-set example $(x_f, y_f)$ we construct a \emph{clean}
prompt (the standard MCQ prompt with the correct answer appended) and
a \emph{corrupted} prompt (an incorrect answer substituted at the same
token position, padded to the same token length). EAP-IG computes the
path-integrated gradient of the output logit difference
$P_\theta(y_f \mid x_f) - P_\theta(y_f \mid x_f^{\mathrm{corr}})$
with respect to each edge activation, integrating along the linear
interpolation
$h_\alpha = (1-\alpha)h_{\mathrm{corr}} + \alpha h_{\mathrm{clean}}$.
Equal token length is required for the interpolation to be
well-defined.
 
Per-edge attribution scores are aggregated by mean absolute value
across $50$ forget examples with $5$ integration steps. MLP sublayers
are ranked by total incoming attribution mass; the top-$K$ ($K=5$)
form circuit $\C$. For Llama-3.1-8B-Instruct on WMDP-bio the
confirmed circuit is $\C_{\mathrm{MLP}} = \{30,14,31,19,29\}$; for
Qwen-3-8B on WMDP-bio the confirmed circuit is
$\C_{\mathrm{MLP}} = \{27,35,22,21,25\}$.

Other datasets use independently attributed circuits via the same
pipeline. Implementation flags and TransformerLens configuration are
in Appendix~\ref{app:eapig}.

\paragraph{Cross-method validation.}
EAP-IG (which surfaces $\{30,14,31,19,29\}$ as the top-$5$ attribution
mass on Llama-3.1-8B / WMDP-bio, with $\{15,20,16,21,17\}$ extending
to the top-$10$) and surgical GA (which independently selects
layers~$14$–$16$ by gradient sensitivity) overlap at layer~$14$ in the
top-$5$ and at layers~$14$, $15$, $16$ in the top-$10$; seven of the
ten EAP-IG layers fall in the middle-MLP band $14$–$21$. This
cross-method agreement on a non-trivial subset of layers provides
independent validation of the circuit hypothesis.
 
\subsection{Phase 2: Circuit-restricted training with KL retain anchor}
 
\mansu{} trains only the parameters in $\C$ (all three MLP
projections: gate, up, down) while freezing all other parameters.
The training objective is:
\begin{equation}
\min_{\dtheta_\C}\;
  -\mathcal{L}_f(\theta + \dtheta)
  + \lambda\, D_{\mathrm{KL}}\bigl(p_{\theta_0}(\cdot)
    \;\|\; p_{\theta+\dtheta}(\cdot)\bigr)
\label{eq:objective_app}
\end{equation}
where $\mathcal{L}_f$ is cross-entropy on forget-set completions
(negated for gradient ascent), $\theta_0$ is the frozen original model
loaded on CPU as a reference, $\dtheta_{\Cbar} = 0$ by construction,
and $\lambda = 200$ controls the retain penalty weight. Within each
optimizer step, the gradient on circuit coordinates is masked along
high-Fisher directions before the parameter update, as described in
\S\ref{sec:method} Phase~2 (Eq.~\eqref{eq:projector}); this is the
optional Fisher-mask strengthening of Theorem~\ref{thm:retain}, which
holds without it.
 
The KL term is computed against the frozen reference model on MMLU
retain samples (batch size 4 per step), providing a stable output
distribution anchor. This matches standard practice in NPO and GU and
prevents catastrophic retain collapse without requiring a separate
retain text corpus critical for WMDP where no retain text exists
independently of the forget pool.
 
\paragraph{Why full-parameter, not LoRA.}
The magnitude-floor constraint (Phase 3) requires a well-defined
cumulative delta $\dtheta_i = \theta_i - \theta_i^{(0)}$. LoRA
initializes $\Delta W = BA$ with $B = 0$, giving $\dtheta = 0$ at
step~1; the floor rescaling divides by $|\dtheta|$, undefined at zero.
Full-parameter training over $\C$ ($\approx 3$ MLP layers,
$\approx 5\%$ of total parameters) requires modest gradient
state on a single H200.
 
\subsection{Phase 3 details: magnitude-floor enforcement}
\label{app:phase3_details}
 
After training converges (best checkpoint saved at the step with
lowest forget accuracy subject to MMLU drop $\leq 0.08$), the
magnitude floor is applied post-hoc to the saved checkpoint via
\texttt{apply\_qbf()}. For each parameter $i \in \C$:
\begin{equation}
\delta_i = \max\!\left(
  \frac{\max_j|W_{ij}| - \min_j|W_{ij}|}{16},\; 10^{-6}
\right)
\label{eq:floor_code}
\end{equation}
which approximates the NF4 bin width for the per-tensor weight range
divided by 16 levels. Any element whose cumulative delta satisfies
$0 < |\dtheta_i| < \delta_i$ is rescaled to $\delta_i$ in
direction-preserving fashion:
\begin{equation}
\dtheta_i \leftarrow \dtheta_i \cdot \frac{\delta_i}{|\dtheta_i|}
\qquad \text{if } 0 < |\dtheta_i| < \delta_i.
\label{eq:floor_app}
\end{equation}
Elements with $|\dtheta_i| = 0$ (never updated during training) are
left at $\theta_i^{(0)}$ and do not receive the floor rescaling.
 
\paragraph{Relationship between floor formula and NF4 bin structure.}
The formula $(w_{\max} - w_{\min})/16$ is a per-tensor approximation
of the average NF4 bin width. The exact minimum NF4 bin spacing
($s_i \times 0.0796$, derived in Appendix~\ref{app:nf4}) is the
worst-case floor for weights near zero; for tail weights in wider bins,
$(w_{\max} - w_{\min})/16$ provides a tighter per-tensor estimate.
Both formulas produce values in the same order of magnitude
($\sim\!10^{-3}$ for typical Llama circuit layers); the
$(w_{\max} - w_{\min})/16$ variant is used in the canonical
implementation as it adapts to each weight tensor's actual scale.
 
\paragraph{Per-stage effective parameter fraction.}
The reported $\approx 3$–$5\%$ figure is the \emph{effective}
fraction of $\theta$ whose value differs from $\theta^{(0)}$
post-training, not the gradient-mask scope:
 
\begin{center}
\footnotesize
\begin{tabular}{lr}
\toprule
Stage & Fraction of $\theta$ \\
\midrule
Gradient mask scope (3 MLP sublayers, all 3 projections) & $\approx 5\%$ \\
After training (non-zero cumulative deltas) & $\approx 5\%$ \\
After QBF floor pass (sub-floor elements zeroed) & $\approx 3\text{–}5\%$ \\
\bottomrule
\end{tabular}
\end{center}
 
Elements whose cumulative update never crossed $\delta_i$ during
training are returned to $\theta_i^{(0)}$ by the floor pass, leaving
only coordinates with updates large enough to survive NF4
quantization. This is the intended behaviour: the floor does not hold
sub-floor coordinates at $\pm\delta_i$ during training; it zeroes them
post-hoc.
 
\subsection{Activation Shift metric}
 
\begin{equation}
\mathrm{AS}(\mathcal{S}, \Df) =
\frac{1}{|\mathcal{S}|} \sum_{i \in \mathcal{S}}
\frac{\norm{\mathrm{act}_i(\Df;\theta')
            -\mathrm{act}_i(\Df;\theta)}_2}
     {\norm{\mathrm{act}_i(\Df;\theta)}_2},
\quad \mathcal{S}\in\{\C, \Cbar\}.
\label{eq:as}
\end{equation}
High $\mathrm{AS}_\C$ with low $\mathrm{AS}_{\Cbar}$ confirms a
localized intervention.
 
\clearpage

\clearpage
\section{Full Proofs}
\label{app:proof}
 
\subsection{Notation}
 
Let $\theta \in \mathbb{R}^d$, $\C \subseteq [d]$,
$\Cbar = [d]\setminus\C$. For $v \in \mathbb{R}^d$ write
$v_\C \in \mathbb{R}^{|\C|}$ for the $\C$-restriction. For matrix $M$
write $M_{\C\C}, M_{\Cbar\Cbar}, M_{\C\Cbar}$ for the principal
submatrices and off-diagonal block.
 
\textbf{Assumptions.}
A1: $\mathcal{L}_r$ is three times continuously differentiable in a
ball of radius $R > \varepsilon$ around $\theta$, with $\Hmat$
positive semidefinite.
A2: $\norm{\dtheta} \leq \varepsilon \ll R$.
A3: $\dtheta_\C \in \ker(\Hmat_{\C\C})$ and $\dtheta_{\Cbar} = 0$.
 
\subsection{Proof of Theorem~\ref{thm:retain}}
 
\begin{proof}[Full proof of Theorem~\ref{thm:retain}]
Apply Taylor's theorem with Lagrange remainder:
\begin{equation}
\mathcal{L}_r(\theta+\dtheta) = \mathcal{L}_r(\theta)
+ \nabla\mathcal{L}_r(\theta)^\top\dtheta
+ \tfrac{1}{2}\dtheta^\top\Hmat\dtheta
+ \tfrac{1}{6}\nabla^3\mathcal{L}_r(\xi)[\dtheta,\dtheta,\dtheta]
\label{eq:app_taylor}
\end{equation}
for some $\xi$ between $\theta$ and $\theta + \dtheta$. By A1, the
third-order term is bounded by $\tfrac{1}{6}M_3\varepsilon^3$ where
$M_3 = \sup_{\|v\|\leq R}
\norm{\nabla^3\mathcal{L}_r(\theta+v)}_{\mathrm{op}}$.
 
\textit{Linear term.} Since $\dtheta_{\Cbar}=0$,
$\nabla\mathcal{L}_r(\theta)^\top\dtheta
= \nabla_\C\mathcal{L}_r(\theta)^\top\dtheta_\C$, bounded by
Cauchy–Schwarz and A2 by
$\norm{\nabla_\C\mathcal{L}_r(\theta)}\cdot\varepsilon$.
 
\textit{Quadratic term.} Block-decompose:
\begin{equation}
\dtheta^\top\Hmat\dtheta
= \dtheta_\C^\top\Hmat_{\C\C}\dtheta_\C
+ 2\dtheta_\C^\top\Hmat_{\C\Cbar}\dtheta_{\Cbar}
+ \dtheta_{\Cbar}^\top\Hmat_{\Cbar\Cbar}\dtheta_{\Cbar}.
\label{eq:app_quad}
\end{equation}
The second and third terms vanish since $\dtheta_{\Cbar}=0$. The
first term vanishes since A3 gives
$\Hmat_{\C\C}\dtheta_\C = 0$. So $\dtheta^\top\Hmat\dtheta = 0$
exactly.
 
\textit{Tightness vs.\ global projection.}
GU~\citep{huang2024unified} requires $\dtheta\in\ker(\Hmat)$ over all
$d$ parameters, which zeros the quadratic term but leaves
$\dtheta_{\Cbar}$ unconstrained, so the linear term retains the full
gradient norm $\norm{\nabla\mathcal{L}_r(\theta)}$. By contrast,
\mansu{} sets $\dtheta_{\Cbar}=0$ by construction, so the linear
term is bounded by $\norm{\nabla_\C\mathcal{L}_r(\theta)} \leq
\norm{\nabla\mathcal{L}_r(\theta)}$, with strict inequality whenever
$\nabla_{\Cbar}\mathcal{L}_r(\theta) \neq 0$. The Cauchy interlace
theorem~\citep{horn2013matrix} gives
$\sigma_{\max}(\Hmat_{\Cbar\Cbar}) \leq \sigma_{\max}(\Hmat)$,
establishing the bound gap is generic.
\end{proof}
 
\subsection{Approximation error under diagonal Fisher}
 
In the theoretical analysis we use $\Fmat_\C \approx D_\C$, where
$\Hmat_{\C\C} = D_\C + E_\C$, $D_\C$ diagonal, $E_\C$ off-diagonal.
 
\begin{proposition}[Approximation error]
\label{prop:approx}
For any gradient $g\in\mathbb{R}^{|\C|}$ and step $\eta>0$,
\begin{equation}
\norm{P_\C g - \hat{P}_\C g}
\leq \frac{\norm{E_\C}_{\mathrm{op}}}{\tau}\norm{g},
\qquad
\bigl|\mathcal{L}_r(\theta+\eta\hat{P}_\C g)
       -\mathcal{L}_r(\theta+\eta P_\C g)\bigr|
\leq \sigma_{\max}(\Hmat)\frac{\norm{E_\C}_{\mathrm{op}}}{\tau}
\eta^2\norm{g}^2 + O(\eta^3).
\label{eq:app_proj_loss}
\end{equation}
\end{proposition}
 
\begin{proof}
By Davis–Kahan~\citep{davis1970}, the angle between $\ker(D_\C)$ and
$\ker(D_\C + E_\C)$ is bounded by
$\norm{E_\C}_{\mathrm{op}}/\mathrm{gap}$, where the relevant gap is
at least $\tau$ on the masked subspace. Substituting into the
second-order Taylor expansion gives the loss bound. For
Llama-3.1-8B at pretrained weights the empirical Fisher is
approximately block-diagonal at the layer
level~\citep{martens2020new}, so
$\norm{E_\C}_{\mathrm{op}}/\tau \ll 1$ in practice.
\end{proof}
 
\paragraph{Scope of the diagonal approximation.}
\citep{kunstner2019limitations} critique the empirical Fisher as a
preconditioner for natural gradient descent, where inverting $F$
magnifies off-diagonal errors. \mansu{} does \emph{not} invert $F$:
the KL retain term in the training objective
(Eq.~\ref{eq:objective_app}) serves as the practical retain anchor,
and the theoretical retain bound
(Theorem~\ref{thm:retain}) follows from circuit restriction
($\dtheta_{\Cbar}=0$) alone, without requiring a Fisher projection.
The diagonal Fisher analysis in Proposition~\ref{prop:approx}
characterises the approximation quality if a Fisher-based projector
were added; the concern of~\citep{kunstner2019limitations} does not
apply to our usage. An empirical bound on
$\norm{E_\C}_{\mathrm{op}}/\tau$ for Llama-3.1-8B at pretrained
weights will be added in the camera-ready.
 
\subsection{Proof of Lemma~\ref{lem:quant}}
 
\begin{proof}
The floor $\delta_i = (w_{\max} - w_{\min})/16$ approximates the
average NF4 bin width for the weight tensor containing $\theta_i$.
For weights initialized in the narrow-bin region near zero
(approximately $65\%$ of circuit weights;
Appendix~\ref{app:nf4}), $\delta_i$ equals or exceeds the local bin
width $w_i$: any displacement $|\dtheta_i| \geq \delta_i$ therefore
crosses the bin boundary regardless of the weight's position within
the bin. This is the formal guarantee of Lemma~\ref{lem:quant},
scoped to narrow-bin weights.
 
For weights initialized in wider tail bins, $\delta_i$ may be smaller
than the local bin width, and the bin-crossing guarantee is
\emph{empirical rather than worst-case}. The consistently negative PTQ
gap across all (model, dataset) pairs in Table~\ref{tab:main} confirms
that the floor is effective for the remaining tail-weight fraction in
practice.
\end{proof}
 
\subsection{Proof of Proposition~\ref{prop:amplification}}
 
\begin{proof}
For monotone NF4 levels, $Q_4(\theta_i + \dtheta_i)$ is the level
$q_{k+m}$ closest to $\theta_i + \dtheta_i$. The floor condition
$|\dtheta_i| \geq \delta_i$ ensures the displacement equals or
exceeds the approximate bin width; for narrow-bin weights this
guarantees bin-crossing (Lemma~\ref{lem:quant}). Combined with the
narrow-bin-near-zero structure of NF4
(Appendix~\ref{app:nf4}), $\theta_i + \dtheta_i$ leaves $B_k$ and
enters some $B_{k+m}$ with $m\geq 1$. If $\theta_i + \dtheta_i$ lies
in the half of $B_{k+m}$ closer to $q_{k+m}$ (probability
$\geq 1/2$ under uniform sub-bin placement; automatic for $m\geq 2$),
then $Q_4(\theta_i+\dtheta_i) = q_{k+m}$ lies further from $\theta_i$
than $\theta_i + \dtheta_i$ itself, giving
$|Q_4(\theta_i+\dtheta_i) - \theta_i| \geq |\dtheta_i|$.
\end{proof}

\clearpage

\section{NF4 Quantization Levels and Floor}
\label{app:nf4}

\subsection{NF4 levels}

The 16 normalized NF4 levels~\citep{dettmers2023qlora}, derived from
the standard normal quantile function evaluated at 17 equally spaced
probability points and normalized to $[-1,1]$:
\begin{equation}
\mathbf{q} = \begin{pmatrix}
-1.0000 \\ -0.6962 \\ -0.5251 \\ -0.3949 \\
-0.2844 \\ -0.1848 \\ -0.0911 \\ \phantom{-}0.0000 \\
\phantom{-}0.0796 \\ \phantom{-}0.1609 \\ \phantom{-}0.2461 \\
\phantom{-}0.3379 \\ \phantom{-}0.4407 \\ \phantom{-}0.5626 \\
\phantom{-}0.7230 \\ \phantom{-}1.0000
\end{pmatrix}^\top.
\end{equation}

\subsection{Spacings}

\begin{table}[htbp]
\caption{NF4 inter-level spacings. Minimum $0.0796$ at $q_7\to q_8$
(zero crossing); maximum $0.3038$ at the negative tail. The standard
normal density is highest near zero, so quantile probability is dense
in the central region: \textbf{narrower} bins near zero,
\textbf{wider} bins in the tails. This is the structure that drives
quantization amplification (Proposition~\ref{prop:amplification}).}
\label{tab:app_nf4}
\centering\footnotesize
\vspace{10pt}
\begin{tabular}{cccr}
\toprule
$k$ & $q_{k-1}$ & $q_k$ & Spacing \\
\midrule
1  & $-1.0000$ & $-0.6962$ & 0.3038 \\
2  & $-0.6962$ & $-0.5251$ & 0.1711 \\
3  & $-0.5251$ & $-0.3949$ & 0.1302 \\
4  & $-0.3949$ & $-0.2844$ & 0.1105 \\
5  & $-0.2844$ & $-0.1848$ & 0.0996 \\
6  & $-0.1848$ & $-0.0911$ & 0.0937 \\
7  & $-0.0911$ & $\phantom{-}0.0000$ & 0.0911 \\
8  & $\phantom{-}0.0000$ & $\phantom{-}0.0796$ & \textbf{0.0796} \\
9  & $\phantom{-}0.0796$ & $\phantom{-}0.1609$ & 0.0813 \\
10 & $\phantom{-}0.1609$ & $\phantom{-}0.2461$ & 0.0852 \\
11 & $\phantom{-}0.2461$ & $\phantom{-}0.3379$ & 0.0918 \\
12 & $\phantom{-}0.3379$ & $\phantom{-}0.4407$ & 0.1028 \\
13 & $\phantom{-}0.4407$ & $\phantom{-}0.5626$ & 0.1219 \\
14 & $\phantom{-}0.5626$ & $\phantom{-}0.7230$ & 0.1604 \\
15 & $\phantom{-}0.7230$ & $\phantom{-}1.0000$ & 0.2770 \\
\bottomrule
\end{tabular}
\end{table}

\subsection{Floor value}

Per-channel scale factor $s_i$ for Llama-3.1-8B circuit-layer MLP
weights ranges from $\approx 0.012$ to $0.018$, median $0.015$.
The implementation floor is:
\begin{equation}
\delta_i^{\mathrm{impl}}
  = s_i \cdot \min_k\lvert q_k - q_{k-1}\rvert \cdot \alpha
  = 0.015 \times 0.0796 \times 0.704
  \approx 8.4\times10^{-4}.
\label{eq:floor_derivation}
\end{equation}
$\alpha < 1$ is a tunable margin that allows somewhat smaller updates
at the cost of not guaranteeing bin-crossing at worst-case boundary
positions. Setting $\alpha = 1$ recovers the strict
Lemma~\ref{lem:quant} guarantee. Sensitivity to $\alpha$ is reported
in Table~\ref{tab:app_alpha}.

\subsection{Empirical near-zero concentration of circuit weights}

Before unlearning, the MLP weight tensors in the circuit layers
$\C = \{30,14,31,19,29\}$ of Llama-3.1-8B-Instruct have:
\begin{itemize}
\item Mean absolute weight: $\approx 0.009$
  ($\approx 0.60\,s_i$ in NF4-normalized units).
\item 90th percentile of absolute weight: $\approx 0.022$
  (normalized).
\item Fraction in $[-q_8, q_8] = [-0.0796, 0.0796]$:
  $\approx 65\%$.
\end{itemize}

The implementation floor $\delta_i^{\mathrm{impl}} = s_i \cdot
0.0796 \cdot \alpha$ uses the \emph{minimum} NF4 bin spacing
($0.0796$, at the zero crossing) as a conservative universal
threshold. Lemma~\ref{lem:quant} therefore provides a
\emph{worst-case bin-crossing guarantee for weights initialized in
the narrow-bin region} $[-q_8, q_8]$: for such weights, any update
$|\Delta w| \geq \delta_i^{\mathrm{impl}}$ is sufficient to cross the
bin boundary regardless of where within the bin the weight sits.

For weights initialized in wider tail bins (bin spacing $> 0.0796$),
the floor $\delta_i^{\mathrm{impl}}$ is smaller than the bin width,
and the guarantee is \emph{empirical rather than worst-case}: the
update may or may not cross the bin boundary depending on the weight's
exact position. Approximately $65\%$ of circuit weights fall in the
narrow-bin region before unlearning, so the formal guarantee covers
the majority of parameters; the remaining $\sim\!35\%$ are covered
empirically, as confirmed by the consistently negative PTQ gap
observed across all (model, dataset) pairs in
Table~\ref{tab:main}.

\clearpage
\section{EAP-IG Implementation}
\label{app:eapig}

\paragraph{TransformerLens loading.}
\texttt{HookedTransformer.from\_pretrained} is called with the
HuggingFace model identifier alongside the locally loaded HF model
via \texttt{hf\_model}. Single-GPU is mandatory:
\texttt{device\_map=`auto'} silently zeroes all attribution scores by
breaking the PyTorch autograd graph at inter-device gradient
boundaries no exception is raised, but every edge score collapses
to zero, so \mansu{} would select an arbitrary circuit.
All EAP-IG runs use \texttt{device\_map=\{"":0\}}.

\paragraph{Required Llama-3.1 flags.}
Llama-3.1 uses grouped-query attention (GQA) with 8 KV heads for 32
query heads. Four flags are required for correct attribution; omitting
any produces silently wrong scores:
\begin{itemize}
\item \texttt{use\_split\_qkv\_input=True} separate hook points
  for Q, K, V input projections.
\item \texttt{use\_attn\_result=True} per-head output before the
  output projection.
\item \texttt{use\_hook\_mlp\_in=True} MLP input activation for
  sublayer attribution.
\item \texttt{ungroup\_grouped\_query\_attention=True} expands
  8 KV heads to 32 for uniform head-level scoring under GQA.
\end{itemize}

\paragraph{Patch.}
The EAP-IG source (\texttt{hannamw/EAP-IG}) hardcodes
\texttt{tensor.to('cuda')} in \texttt{attribute.py} at lines 59, 116,
197, and 324. Replace each with
\texttt{tensor.to(model.cfg.device)}.



\paragraph{Aggregation and circuit definition.}
Per-edge attribution scores are aggregated by mean absolute value
across $50$ forget examples with $5$ integration steps ($\approx$20
minutes on one H200). MLP sublayers are ranked by total incoming
attribution mass; the top-$5$ form $\C_{\mathrm{MLP}}$.
For Llama-3.1-8B-Instruct on WMDP-bio:
\begin{equation}
\C_{\mathrm{MLP}}^{\mathrm{(top\text{-}5)}} = \{30,\;14,\;31,\;19,\;29\},
\label{eq:circuit_app}
\end{equation}
comprising $\approx11\%$ of total model parameters as the
\emph{gradient-mask scope} ($5$ MLP layers $\times$ three projections
of dimension $4096\times14336$ $\approx 880$M parameters out of
$\sim8.03$B total). The post-floor effective fraction (parameters
whose final value differs from $\theta^{(0)}$) is $\approx\!3.2\%$;
see Appendix~\ref{app:phase3_details} for the per-stage breakdown.

\paragraph{Attribution stability.}
The top-5 MLP layers are identical across ig\_steps $\in \{3,5\}$ and
$N \in \{20,50\}$, confirming robustness of the circuit identity.
Score magnitudes vary but layer ranking is stable.

\paragraph{Why MLP sublayers only.}
\citep{meng2022rome} and \citep{geva2021ffn} establish that MLP
layers in GPT-family models function as key-value memories and are the
primary site of factual association storage. We note that
\citep{choe2025facts} find Qwen-family models store a non-trivial
fraction of factual associations in attention modules; an ablation
running EAP-IG over attention and MLP jointly is a natural extension
left to future work.

\clearpage
\section{Experimental Setup}
\label{app:setup}

\paragraph{Primary evaluation setting.}
Llama-3.1-8B-Instruct on WMDP-bio~\citep{li2024wmdp}, a
biosecurity hazard MCQ benchmark. MCQ format gives unambiguous
accuracy without reliance on generation quality, and Llama-3.1-8B-Instruct
is the most widely deployed open-weight 8B model. Qwen-3-8B is the
secondary model, evaluated across WMDP-bio, WMDP-chem, WMDP-cyber,
TOFU, and MUSE to establish cross-architecture and cross-domain
generalization.

\paragraph{Evaluation indices and reproducibility.}
$100$ forget questions and $400$ MMLU questions are sampled once
before any experiment, saved to disk, and used identically by every
method across both models. Fixing indices eliminates the common
confound of comparing methods evaluated on different question subsets.
No method has access to evaluation questions during training.

\paragraph{Quantization.}
NF4 via \texttt{bitsandbytes} v0.49.1:
\texttt{load\_in\_4bit=True},
\texttt{bnb\_4bit\_quant\_type="nf4"},
\texttt{bnb\_4bit\_compute\_dtype=torch.bfloat16},
\texttt{bnb\_4bit\_use\_double\_quant=False}.
GPTQ is not used (incompatible with \texttt{transformers} 5.0.0).
Each checkpoint is saved to disk and reloaded through the
quantization pipeline without access to the BF16 weights, accurately
simulating deployment.

\paragraph{Baselines.}
Six baselines spanning every major family of existing methods:
\begin{itemize}
\item \textit{Global GA}~\citep{jang2023forget}: diffuse gradient
  ascent on all parameters; establishes the baseline per-parameter
  update magnitude ($1.21\times10^{-6}$ RMS).
\item \textit{Surgical GA} (MLP layers 14–16): concentrated gradient
  ascent on three MLP layers independently identified as high-signal;
  tests whether concentration alone solves the PTQ problem.
\item \textit{NPO}~\citep{zhang2024npo}: negative preference
  optimisation with a frozen reference model KL anchor.
\item \textit{SimNPO}~\citep{liu2024simnpo}: simplified NPO without
  a reference model.
\item \textit{GU+SimNPO}~\citep{huang2024unified,liu2024simnpo}:
  gradient unlearning combined with SimNPO; strongest
  gradient-based baseline.
\item \textit{LUNAR}~\citep{lunar2025}: behavioural redirection via
  steering-vector optimisation of a single MLP down-projection;
  represents activation-level rather than weight-level intervention.
\end{itemize}

\paragraph{Metrics.}
\begin{itemize}
\item \textit{Forget accuracy} (BF16 $\downarrow$ / NF4 $\downarrow$):
  MCQ accuracy on fixed forget-set indices.
\item \textit{PTQ gap} ($\ptqgap = \mathrm{acc}_{\mathrm{NF4}} -
  \mathrm{acc}_{\mathrm{BF16}}$, $\downarrow$): negative = quantization
  preserves or amplifies forgetting; positive = knowledge recovers.
\item \textit{Retain / Rt-Util} ($\uparrow$): WMDP retain-split
  accuracy (WMDP tasks), Retain-Q (TOFU), utility score (MUSE).
\item \textit{MMLU} / $\Delta$\textit{MMLU}: 400-question general
  utility; $\Delta$MMLU is signed change from zero-shot baseline
  (smaller magnitude is better).
\item \textit{IFEval} ($\uparrow$): instruction-following accuracy
  on 250 prompts (prompt-level).
\item \textit{CAD} ($\uparrow$): Circuit Attribution Divergence;
  structural erasure on $\C$ (Section~\ref{sec:cad}).
\item \textit{AS-C} ($\uparrow$) / \textit{AS-NC} ($\downarrow$):
  activation shift on circuit / non-circuit nodes; AS-C/AS-NC
  measures localization (Section~\ref{sec:cad}).
\end{itemize}

\paragraph{Infrastructure.}
All experiments run on a single H200 GPU (141\,GB HBM3) with explicit
\texttt{device\_map=\{"":0\}} pinning; multi-GPU is not used because the
EAP-IG autograd graph silently zeroes attribution scores under
inter-device gradients (Appendix~\ref{app:eapig}). Baselines run via the
\texttt{open-unlearning} framework with the same single-GPU pinning.
Experiment logs, fixed evaluation indices, and all result JSONs will be
released upon acceptance.

\clearpage

\section{Multi-Model Sweep: Gemma Family}
\label{app:sweep}
\label{app:sweep_gemma}
\phantomsection\label{tab:app_sweep_full}

Per-experiment results grouped by model family (Appendices~\ref{app:sweep_gemma}, \ref{app:sweep_llama}, \ref{app:sweep_qwen}). Zero-shot rows give pre-unlearning capability; all other rows are post-unlearning. Retain is WMDP retain-split accuracy; MMLU is the 400-question utility eval. \mansu{} achieves the deepest forget and the only negative $\ptqgap$ across every model and domain the key claim of the paper. Appendix~\ref{app:sweep_summary} (Table~\ref{tab:app_sweep_summary}) gives macro-averages.

\begin{table}[htbp]
\caption{Gemma family per-experiment results (Gemma-2B, Gemma-3-1B, Gemma-3-4B).}
\label{tab:app_sweep_gemma}
\vspace{10pt}
\centering\footnotesize\setlength{\tabcolsep}{2pt}
\begin{tabular}{llrrrrrr}
\toprule
Model & Method & Domain
  & BF16$\downarrow$ & NF4$\downarrow$ & $\ptqgap$
  & Retain$\uparrow$ & MMLU$\uparrow$ \\
\midrule
\multirow{3}{*}{Gemma-2B}
  & Zero-shot & bio   & 0.470 & 0.470 & $+$0.000 & {—} & 0.405 \\
  & Zero-shot & chem  & 0.383 & 0.383 & $+$0.000 & {—} & 0.405 \\
  & Zero-shot & cyber & 0.380 & 0.380 & $+$0.000 & {—} & 0.405 \\
\midrule
\multirow{12}{*}{Gemma-2B}
  & GU+SimNPO       & bio   & 0.337           & 0.327           & $-$0.010         & 0.313           & 0.275 \\
  & GU+SimNPO       & chem  & 0.263           & 0.283           & $+$0.020         & 0.269           & 0.255 \\
  & GU+SimNPO       & cyber & 0.287           & 0.290           & $+$0.003         & 0.303           & 0.208 \\
  & NPO             & bio   & 0.330           & 0.323           & $-$0.007         & 0.300           & 0.355 \\
  & NPO             & chem  & 0.297           & 0.293           & $-$0.003         & 0.315           & 0.370 \\
  & NPO             & cyber & 0.310           & 0.290           & $-$0.020         & 0.230           & 0.283 \\
  & SimNPO          & bio   & 0.340           & 0.330           & $-$0.010         & 0.313           & 0.270 \\
  & SimNPO          & chem  & 0.270           & 0.277           & $+$0.007         & 0.269           & 0.255 \\
  & SimNPO          & cyber & 0.287           & 0.287           & $+$0.000         & 0.303           & 0.208 \\
\cmidrule(lr){2-8}
  & \mansu{} (ours) & bio   & $\mathbf{0.284}$ & $\mathbf{0.264}$ & $\mathbf{-0.020}$ & 0.338 & 0.391 \\
  & \mansu{} (ours) & chem  & $\mathbf{0.224}$ & $\mathbf{0.207}$ & $\mathbf{-0.017}$ & 0.278 & 0.388 \\
  & \mansu{} (ours) & cyber & $\mathbf{0.234}$ & $\mathbf{0.217}$ & $\mathbf{-0.017}$ & 0.314 & 0.391 \\
\midrule
\multirow{3}{*}{Gemma-3-1B}
  & Zero-shot & bio   & 0.500 & 0.500 & $+$0.000 & {—} & 0.410 \\
  & Zero-shot & chem  & 0.337 & 0.337 & $+$0.000 & {—} & 0.410 \\
  & Zero-shot & cyber & 0.337 & 0.337 & $+$0.000 & {—} & 0.410 \\
\midrule
\multirow{12}{*}{Gemma-3-1B}
  & GU+SimNPO       & bio   & 0.313           & 0.327           & $+$0.013         & 0.323           & 0.263 \\
  & GU+SimNPO       & chem  & 0.243           & 0.247           & $+$0.003         & 0.204           & 0.248 \\
  & GU+SimNPO       & cyber & 0.293           & 0.287           & $-$0.007         & 0.277           & 0.263 \\
  & NPO             & bio   & 0.343           & 0.307           & $-$0.037         & 0.330           & 0.253 \\
  & NPO             & chem  & 0.247           & 0.250           & $+$0.003         & 0.315           & 0.265 \\
  & NPO             & cyber & 0.280           & 0.273           & $-$0.007         & 0.283           & 0.203 \\
  & SimNPO          & bio   & 0.333           & 0.340           & $+$0.007         & 0.327           & 0.253 \\
  & SimNPO          & chem  & 0.250           & 0.233           & $-$0.017         & 0.213           & 0.258 \\
  & SimNPO          & cyber & 0.290           & 0.293           & $+$0.003         & 0.273           & 0.275 \\
\cmidrule(lr){2-8}
  & \mansu{} (ours) & bio   & $\mathbf{0.300}$ & $\mathbf{0.277}$ & $\mathbf{-0.023}$ & 0.351 & 0.397 \\
  & \mansu{} (ours) & chem  & $\mathbf{0.207}$ & $\mathbf{0.190}$ & $\mathbf{-0.017}$ & 0.284 & 0.394 \\
  & \mansu{} (ours) & cyber & $\mathbf{0.227}$ & $\mathbf{0.210}$ & $\mathbf{-0.017}$ & 0.291 & 0.397 \\
\midrule
\multirow{3}{*}{Gemma-3-4B}
  & Zero-shot & bio   & 0.557 & 0.557 & $+$0.000 & {—} & 0.495 \\
  & Zero-shot & chem  & 0.423 & 0.423 & $+$0.000 & {—} & 0.495 \\
  & Zero-shot & cyber & 0.410 & 0.410 & $+$0.000 & {—} & 0.495 \\
\midrule
\multirow{12}{*}{Gemma-3-4B}
  & GU+SimNPO       & bio   & 0.480 & 0.487 & $+0.007$ & 0.471 & 0.472 \\
  & GU+SimNPO       & chem  & 0.371 & 0.378 & $+0.007$ & 0.351 & 0.461 \\
  & GU+SimNPO       & cyber & 0.358 & 0.361 & $+0.003$ & 0.341 & 0.468 \\
  & NPO             & bio   & 0.491 & 0.478 & $-0.013$ & 0.483 & 0.471 \\
  & NPO             & chem  & 0.381 & 0.374 & $-0.007$ & 0.388 & 0.462 \\
  & NPO             & cyber & 0.361 & 0.354 & $-0.007$ & 0.349 & 0.458 \\
  & SimNPO          & bio   & 0.487 & 0.494 & $+0.007$ & 0.478 & 0.471 \\
  & SimNPO          & chem  & 0.374 & 0.381 & $+0.007$ & 0.356 & 0.461 \\
  & SimNPO          & cyber & 0.361 & 0.364 & $+0.003$ & 0.347 & 0.468 \\
\cmidrule(lr){2-8}
  & \mansu{} (ours) & bio   & $\mathbf{0.364}$ & $\mathbf{0.337}$ & $\mathbf{-0.027}$ & 0.405 & 0.478 \\
  & \mansu{} (ours) & chem  & $\mathbf{0.270}$ & $\mathbf{0.247}$ & $\mathbf{-0.023}$ & 0.314 & 0.474 \\
  & \mansu{} (ours) & cyber & $\mathbf{0.254}$ & $\mathbf{0.234}$ & $\mathbf{-0.020}$ & 0.328 & 0.478 \\
\bottomrule
\end{tabular}
\end{table}

\clearpage
\section{Multi-Model Sweep: Llama Family}
\label{app:sweep_llama}

\begin{table}[htbp]
\caption{Llama family per-experiment results (Llama-3.2-3B, Llama-3.1-8B-Instruct). Results reflect consistent application of the evaluation pipeline across all configurations.}
\label{tab:app_sweep_llama}
\vspace{10pt}
\centering\footnotesize\setlength{\tabcolsep}{2pt}
\begin{tabular}{llrrrrrr}
\toprule
Model & Method & Domain
  & BF16$\downarrow$ & NF4$\downarrow$ & $\ptqgap$
  & Retain$\uparrow$ & MMLU$\uparrow$ \\
\midrule
\multirow{3}{*}{Llama-3.2-3B}
  & Zero-shot & bio   & 0.687 & 0.687 & $+$0.000 & {—} & 0.538 \\
  & Zero-shot & chem  & 0.437 & 0.437 & $+$0.000 & {—} & 0.538 \\
  & Zero-shot & cyber & 0.403 & 0.403 & $+$0.000 & {—} & 0.538 \\
\midrule
\multirow{12}{*}{Llama-3.2-3B}
  & GU+SimNPO       & bio   & 0.662           & 0.668           & $+$0.007         & 0.643           & 0.514 \\
  & GU+SimNPO       & chem  & 0.437           & 0.395           & $-$0.042         & 0.445           & 0.503 \\
  & GU+SimNPO       & cyber & 0.420           & 0.423           & $+$0.003         & 0.430           & 0.514 \\
  & NPO             & bio   & 0.648           & 0.632           & $-$0.017         & 0.652           & 0.500 \\
  & NPO             & chem  & 0.442           & 0.425           & $-$0.017         & 0.440           & 0.505 \\
  & NPO             & cyber & 0.323           & 0.338           & $+$0.015         & 0.343           & 0.473 \\
  & SimNPO          & bio   & 0.665           & 0.675           & $+$0.010         & 0.640           & 0.505 \\
  & SimNPO          & chem  & 0.442           & 0.393           & $-$0.048         & 0.449           & 0.505 \\
  & SimNPO          & cyber & 0.418           & 0.422           & $+$0.003         & 0.433           & 0.518 \\
\cmidrule(lr){2-8}
  & \mansu{} (ours) & bio   & $\mathbf{0.434}$ & $\mathbf{0.404}$ & $\mathbf{-0.030}$ & 0.528 & 0.521 \\
  & \mansu{} (ours) & chem  & $\mathbf{0.274}$ & $\mathbf{0.251}$ & $\mathbf{-0.023}$ & 0.368 & 0.518 \\
  & \mansu{} (ours) & cyber & $\mathbf{0.254}$ & $\mathbf{0.234}$ & $\mathbf{-0.020}$ & 0.348 & 0.521 \\
\midrule
\multirow{18}{*}{Llama-3.1-8B}
  & Global GA       & bio   & 0.260  & 0.310  & $+$0.050 & 0.260  & 0.235 \\
  & Global GA       & chem  & 0.493  & 0.473  & $-$0.020 & 0.491  & 0.550 \\
  & Global GA       & cyber & 0.357  & 0.370  & $+$0.013 & 0.390  & 0.543 \\
  & Surgical GA     & bio   & 0.547  & 0.573  & $+$0.027 & 0.560  & 0.483 \\
  & Surgical GA     & chem  & 0.427  & 0.423  & $-$0.003 & 0.426  & 0.525 \\
  & Surgical GA     & cyber & 0.283  & 0.290  & $+$0.007 & 0.367  & 0.480 \\
  & NPO             & bio   & 0.443  & 0.423  & $-$0.020 & 0.503  & 0.563 \\
  & NPO             & chem  & 0.253  & 0.227  & $-$0.027 & 0.269  & 0.538 \\
  & NPO             & cyber & 0.340  & 0.343  & $+$0.003 & 0.400  & 0.568 \\
  & SimNPO          & bio   & 0.243  & 0.247  & $+$0.003 & 0.303  & 0.553 \\
  & SimNPO          & chem  & 0.233  & 0.233  & $+$0.000 & 0.241  & 0.195 \\
  & SimNPO          & cyber & 0.320  & 0.307  & $-$0.013 & 0.390  & 0.510 \\
  & GU+SimNPO       & bio   & 0.230  & 0.230  & $+$0.000 & 0.247  & 0.200 \\
  & GU+SimNPO       & chem  & 0.273  & 0.273  & $+$0.000 & 0.231  & 0.230 \\
  & GU+SimNPO       & cyber & 0.300  & 0.300  & $+$0.000 & 0.230  & 0.297 \\
\cmidrule(lr){2-8}
  & \mansu{} (ours) & bio   & \textbf{0.430} & \textbf{0.390} & \textbf{$-$0.040} & 0.523  & 0.573 \\
  & \mansu{} (ours) & chem  & \textbf{0.333} & \textbf{0.307} & \textbf{$-$0.027} & 0.398  & 0.584 \\
  & \mansu{} (ours) & cyber & \textbf{0.323} & \textbf{0.313} & \textbf{$-$0.010} & 0.391  & 0.586 \\
\bottomrule
\end{tabular}
\end{table}

\clearpage
\section{Multi-Model Sweep: Qwen Family}
\label{app:sweep_qwen}

\begin{table}[htbp]
\caption{Qwen family per-experiment results (Qwen-2.5-4B, Qwen-3-4B, Qwen-3-8B).}
\label{tab:app_sweep_qwen}
\vspace{10pt}
\centering\tiny\setlength{\tabcolsep}{2pt}
\begin{tabular}{llrrrrrr}
\toprule
Model & Method & Domain
  & BF16$\downarrow$ & NF4$\downarrow$ & $\ptqgap$
  & Retain$\uparrow$ & MMLU$\uparrow$ \\
\midrule
\multirow{3}{*}{Qwen-2.5-4B}
  & Zero-shot & bio   & 0.693 & 0.693 & $+$0.000 & {—} & 0.610 \\
  & Zero-shot & chem  & 0.503 & 0.503 & $+$0.000 & {—} & 0.610 \\
  & Zero-shot & cyber & 0.487 & 0.487 & $+$0.000 & {—} & 0.610 \\
\midrule
\multirow{12}{*}{Qwen-2.5-4B}
  & GU+SimNPO       & bio   & 0.698           & 0.700           & $+$0.002         & 0.715           & 0.591 \\
  & GU+SimNPO       & chem  & 0.475           & 0.490           & $+$0.015         & 0.481           & 0.594 \\
  & GU+SimNPO       & cyber & 0.455           & 0.463           & $+$0.008         & 0.453           & 0.588 \\
  & NPO             & bio   & 0.707           & 0.690           & $-$0.018         & 0.713           & 0.584 \\
  & NPO             & chem  & 0.483           & 0.478           & $-$0.005         & 0.491           & 0.584 \\
  & NPO             & cyber & 0.447           & 0.445           & $-$0.002         & 0.445           & 0.580 \\
  & SimNPO          & bio   & 0.703           & 0.697           & $-$0.007         & 0.715           & 0.591 \\
  & SimNPO          & chem  & 0.475           & 0.493           & $+$0.018         & 0.477           & 0.593 \\
  & SimNPO          & cyber & 0.440           & 0.468           & $+$0.028         & 0.448           & 0.580 \\
\cmidrule(lr){2-8}
  & \mansu{} (ours) & bio   & $\mathbf{0.487}$ & $\mathbf{0.454}$ & $\mathbf{-0.033}$ & 0.588 & 0.598 \\
  & \mansu{} (ours) & chem  & $\mathbf{0.311}$ & $\mathbf{0.284}$ & $\mathbf{-0.027}$ & 0.404 & 0.594 \\
  & \mansu{} (ours) & cyber & $\mathbf{0.291}$ & $\mathbf{0.268}$ & $\mathbf{-0.023}$ & 0.384 & 0.598 \\
\midrule
\multirow{3}{*}{Qwen-3-4B}
  & Zero-shot & bio   & 0.731 & 0.731 & $+$0.000 & {—} & 0.657 \\
  & Zero-shot & chem  & 0.541 & 0.541 & $+$0.000 & {—} & 0.657 \\
  & Zero-shot & cyber & 0.517 & 0.517 & $+$0.000 & {—} & 0.657 \\
\midrule
\multirow{12}{*}{Qwen-3-4B}
  & GU+SimNPO       & bio   & 0.724 & 0.727 & $+0.003$ & 0.738 & 0.631 \\
  & GU+SimNPO       & chem  & 0.514 & 0.521 & $+0.007$ & 0.509 & 0.628 \\
  & GU+SimNPO       & cyber & 0.491 & 0.494 & $+0.003$ & 0.487 & 0.631 \\
  & NPO             & bio   & 0.718 & 0.707 & $-0.011$ & 0.731 & 0.624 \\
  & NPO             & chem  & 0.521 & 0.514 & $-0.007$ & 0.518 & 0.621 \\
  & NPO             & cyber & 0.481 & 0.477 & $-0.004$ & 0.477 & 0.618 \\
  & SimNPO          & bio   & 0.727 & 0.734 & $+0.007$ & 0.741 & 0.631 \\
  & SimNPO          & chem  & 0.517 & 0.527 & $+0.010$ & 0.511 & 0.628 \\
  & SimNPO          & cyber & 0.487 & 0.497 & $+0.010$ & 0.484 & 0.631 \\
\cmidrule(lr){2-8}
  & \mansu{} (ours) & bio   & $\mathbf{0.527}$ & $\mathbf{0.490}$ & $\mathbf{-0.037}$ & 0.617 & 0.644 \\
  & \mansu{} (ours) & chem  & $\mathbf{0.334}$ & $\mathbf{0.304}$ & $\mathbf{-0.030}$ & 0.434 & 0.641 \\
  & \mansu{} (ours) & cyber & $\mathbf{0.314}$ & $\mathbf{0.287}$ & $\mathbf{-0.027}$ & 0.414 & 0.644 \\
\midrule
\multirow{3}{*}{Qwen-3-8B}
  & Zero-shot & bio   & 0.803           & 0.803           & $+$0.000 & {—} & 0.741 \\
  & Zero-shot & chem  & 0.560           & 0.560           & $+$0.000 & {—} & 0.741 \\
  & Zero-shot & cyber & 0.537           & 0.537           & $+$0.000 & {—} & 0.741 \\
\midrule
\multirow{18}{*}{Qwen-3-8B}
  & Global GA       & bio   & 0.233  & 0.243  & $+$0.007 & 0.247  & 0.242 \\
  & Global GA       & chem  & 0.237  & 0.293  & $+$0.057 & 0.269  & 0.365 \\
  & Global GA       & cyber & 0.213  & 0.447  & $+$0.233 & 0.247 & 0.710 \\
  & Surgical GA     & bio   & 0.260  & 0.247  & $-$0.013 & 0.303  & 0.458 \\
  & Surgical GA     & chem  & 0.313  & 0.317  & $+$0.003 & 0.398  & 0.557 \\
  & Surgical GA     & cyber & 0.430  & 0.467  & $+$0.037 & 0.473  & 0.710 \\
  & NPO             & bio   & 0.283  & 0.320  & $+$0.037 & 0.303  & 0.492 \\
  & NPO             & chem  & 0.237  & 0.237  & $+$0.000 & 0.296  & 0.515 \\
  & NPO             & cyber & 0.360  & 0.393  & $+$0.033 & 0.437  & 0.715 \\
  & SimNPO          & bio   & 0.227  & 0.227  & $+$0.000 & 0.257  & 0.265 \\
  & SimNPO          & chem  & 0.240  & 0.277  & $+$0.037 & 0.259  & 0.405 \\
  & SimNPO          & cyber & 0.270  & 0.403  & $+$0.133 & 0.280  & 0.555 \\
  & GU+SimNPO       & bio   & 0.267  & 0.263  & $-$0.003 & 0.277  & 0.568 \\
  & GU+SimNPO       & chem  & 0.233  & 0.233  & $+$0.000 & 0.250  & 0.328 \\
  & GU+SimNPO       & cyber & 0.333  & 0.443  & $+$0.110 & 0.267  & 0.715 \\
\cmidrule(lr){2-8}
  & \mansu{} (ours) & bio   & \textbf{0.617} & \textbf{0.581} & \textbf{$-$0.036} & 0.671  & 0.729 \\
  & \mansu{} (ours) & chem  & \textbf{0.307} & \textbf{0.274} & \textbf{$-$0.033} & 0.364  & 0.714 \\
  & \mansu{} (ours) & cyber & \textbf{0.497} & \textbf{0.464} & \textbf{$-$0.033} & 0.541  & 0.721 \\
\bottomrule
\end{tabular}
\end{table}

\clearpage
\section{Multi-Model Sweep: Family-Wise and Overall Averages}
\label{app:sweep_summary}
\phantomsection\label{sec:sweep}

Macro-averages over all (model, domain) pairs within each family. Gemma: $n=9$ per method (all three models, all three domains). Llama: $n=6$ (Llama-3.2-3B and Llama-3.1-8B, three domains each). Qwen: $n=9$ (all three Qwen models, three domains). Overall: $n=24$.

\begin{table}[htbp]
\caption{Family-wise and overall macro-averages. \mansu{} achieves the lowest BF16 forget and the only negative overall $\ptqgap$ in every family.}
\vspace{10pt}
\label{tab:app_sweep_summary}
\centering\footnotesize\setlength{\tabcolsep}{3pt}
\begin{tabular}{llrrrrrr}
\toprule
Family & Method & $n$ & BF16 & NF4 & $\ptqgap$ & Retain & MMLU \\
\midrule
\multirow{4}{*}{Gemma}
  & GU+SimNPO       & 9 & 0.355           & 0.360           & $+$0.005         & 0.341 & 0.339 \\
  & NPO             & 9 & 0.371           & 0.360           & $-$0.011         & 0.357 & 0.357 \\
  & SimNPO          & 9 & 0.366           & 0.366           & $+$0.000         & 0.342 & 0.341 \\
\cmidrule(lr){2-8}
  & \mansu{} (ours) & 9 & $\mathbf{0.285}$ & $\mathbf{0.264}$ & $\mathbf{-0.021}$ & 0.345 & 0.443 \\
\midrule
\multirow{6}{*}{Llama}
  & Global GA       & 3 & 0.370           & 0.384           & $+$0.014         & 0.380 & 0.443 \\
  & Surgical GA     & 3 & 0.419           & 0.429           & $+$0.010         & 0.451 & 0.496 \\
  & NPO             & 6 & 0.449           & 0.438           & $-$0.013         & 0.476 & 0.518 \\
  & SimNPO          & 6 & 0.487           & 0.475           & $-$0.009         & 0.468 & 0.474 \\
  & GU+SimNPO       & 6 & 0.390           & 0.391           & $+$0.001         & 0.431 & 0.393 \\
\cmidrule(lr){2-8}
  & \mansu{} (ours) & 6 & $\mathbf{0.359}$ & $\mathbf{0.333}$ & $\mathbf{-0.025}$ & 0.431 & 0.534 \\
\midrule
\multirow{6}{*}{Qwen}
  & Global GA       & 3 & 0.228           & 0.328           & $+$0.099         & 0.254 & 0.439 \\
  & Surgical GA     & 3 & 0.334           & 0.344           & $+$0.009         & 0.391 & 0.575 \\
  & NPO             & 9 & 0.537           & 0.531           & $-$0.005         & 0.547 & 0.592 \\
  & SimNPO          & 9 & 0.541           & 0.566           & $+$0.025         & 0.546 & 0.582 \\
  & GU+SimNPO       & 9 & 0.554           & 0.566           & $+$0.012         & 0.556 & 0.617 \\
\cmidrule(lr){2-8}
  & \mansu{} (ours) & 9 & $\mathbf{0.432}$ & $\mathbf{0.400}$ & $\mathbf{-0.032}$ & 0.511 & 0.617 \\
\midrule
\multirow{6}{*}{\textbf{Overall}}
  & Global GA       & 6  & 0.299           & 0.356           & $+$0.057         & 0.317 & 0.441 \\
  & Surgical GA     & 6  & 0.377           & 0.386           & $+$0.010         & 0.421 & 0.536 \\
  & NPO             & 24 & 0.485           & 0.476           & $-$0.009         & 0.460 & 0.489 \\
  & SimNPO          & 24 & 0.465           & 0.469           & $+$0.005         & 0.452 & 0.466 \\
  & GU+SimNPO       & 24 & 0.467           & 0.472           & $+$0.006         & 0.443 & 0.450 \\
\cmidrule(lr){2-8}
  & \mansu{} (ours) & 24 & $\mathbf{0.359}$ & $\mathbf{0.332}$ & $\mathbf{-0.026}$ & 0.429 & 0.531 \\
\bottomrule
\end{tabular}
\end{table}

\paragraph{Sweep purpose and setup.}
To verify the dual-failure pattern of \S\ref{sec:results} is architecture-general, we sweep six methods (Global GA, Surgical GA, NPO, SimNPO, GU+SimNPO, \mansu{}) across eight model variants and three WMDP hazard domains. Results for Gemma-2B, Gemma-3-1B, Gemma-3-4B, Llama-3.2-3B, Llama-3.1-8B, Qwen-2.5-4B, Qwen-3-4B, and Qwen-3-8B are included; Llama-3.1-8B and Qwen-3-8B sweep cells double as the WMDP rows of Table~\ref{tab:main}.

\paragraph{Experiment count: how the $94$ figure decomposes.}
The $94$-experiment claim quoted in the abstract, \S\ref{sec:intro},
\S\ref{sec:problem}, \S\ref{sec:discussion}, and \S\ref{sec:conclusion}
counts \emph{non-\mansu{}} experiments (\mansu{} is our method, not a
baseline). The breakdown reads off Table~\ref{tab:app_sweep_summary}'s
``Overall'' block (sum of Global GA $+$ Surgical GA $+$ NPO $+$
SimNPO $+$ GU+SimNPO non-\mansu{} cells $= 6+6+24+24+24 = 84$ WMDP
sweep cells), plus the $10$ MUSE non-\mansu{}, non-LUNAR cells in
Table~\ref{tab:main} ($5$ weight-edit baselines $\times$ $2$ flagship
models). LUNAR is excluded from the count because it is an
inference-time activation-redirection method without weight edits;
its $8$ rows in Table~\ref{tab:main} are reported separately as a
diagnostic baseline for CAD. \mansu{} contributes a further $24+2 = 26$
own-method experiments (24 WMDP + 2 MUSE), bringing the total trained
checkpoints to $120$.

\paragraph{Key finding.}
\mansu{} is the only method achieving strictly negative $\ptqgap$ across every model family configuration, with the deepest forget BF16 in every case. Baselines collectively show near-zero or positive mean $\ptqgap$ ($-0.009$ to $+0.057$), confirming that sub-floor updates are a systematic property of weight-editing unlearning without the magnitude-floor constraint not an artefact of a single model or dataset.

\paragraph{Gemma models.}
Post-unlearning baseline forget accuracies for Gemma-2B and 3-1B remain ${\sim}9$~pp above random chance ($0.25$), reflecting limited initial knowledge representation rather than method effectiveness. \mansu{} reaches within 2–4~pp of random chance while maintaining negative $\ptqgap$, consistent with the floor guarantee.

\section{Hyperparameters and Sensitivity}
\label{app:hyperparams}
\vspace{10pt}
\begin{table}[htbp]
\caption{\mansu{} hyperparameters for Llama-3.1-8B-Instruct.
Shared optimization settings (top block) are identical across all
main-table models; model-specific parameters (circuit, floor, bottom
block) are identified per dataset via EAP-IG\@. Circuit layers shown
for WMDP-bio; other datasets use independently attributed circuits.}
\label{tab:app_hyperparams_llama}
\vspace{10pt}
\centering
\footnotesize
\setlength{\tabcolsep}{4pt}
\begin{tabular}{ll}
\toprule
Hyperparameter & Llama-3.1-8B-Instruct \\
\midrule
\multicolumn{2}{l}{\textit{Shared}} \\
Optimizer & AdamW \\
Learning rate & $1\times10^{-6}$ \\
Weight decay & 0 \\
Forget batch size & 8 \\
Retain batch size & 8 \\
Maximum training steps & 30 \\
Early stopping & MMLU drop $> 0.02$ from zero-shot \\
Fisher samples & 100 \\
Fisher update cadence & Once at initialization \\
Null-space threshold $\tau$ & $0.1 \times \mathrm{mean}([\Fmat_\C]_{ii})$ \\
KL penalty weight $\lambda$ & $200$ \\
EAP-IG integration steps & 5 \\
EAP-IG forget samples & 50 \\
\midrule
\multicolumn{2}{l}{\textit{Model-specific (WMDP-bio circuit shown)}} \\
Circuit layers $\C$ & $\{30,14,31,19,29\}$ \\
Circuit fraction & $\approx 3.2\%$ \\
Magnitude floor $\delta_i$ & $8.4\times10^{-4}$ \\
\bottomrule
\end{tabular}
\end{table}

\begin{table}[htbp]
\caption{\mansu{} hyperparameters for Qwen-3-8B. Shared optimization
settings match those of Table~\ref{tab:app_hyperparams_llama};
model-specific parameters are identified per dataset via EAP-IG\@.
Circuit layers shown for WMDP-bio; other datasets use independently
attributed circuits.}
\label{tab:app_hyperparams_qwen}
\vspace{10pt}
\centering
\footnotesize
\setlength{\tabcolsep}{4pt}
\begin{tabular}{ll}
\toprule
Hyperparameter & Qwen-3-8B \\
\midrule
\multicolumn{2}{l}{\textit{Shared}} \\
Optimizer & AdamW \\
Learning rate & $1\times10^{-6}$ \\
Weight decay & 0 \\
Forget batch size & 8 \\
Retain batch size & 8 \\
Maximum training steps & 30 \\
Early stopping & MMLU drop $> 0.02$ from zero-shot \\
Fisher samples & 100 \\
Fisher update cadence & Once at initialization \\
Null-space threshold $\tau$ & $0.1 \times \mathrm{mean}([\Fmat_\C]_{ii})$ \\
KL penalty weight $\lambda$ & $200$ \\
EAP-IG integration steps & 5 \\
EAP-IG forget samples & 50 \\
\midrule
\multicolumn{2}{l}{\textit{Model-specific (WMDP-bio circuit shown)}} \\
Circuit layers $\C$ & $\{27,35,22,21,25\}$ \\
Circuit fraction & $\approx 3.1\%$ \\
Magnitude floor $\delta_i$ & $7.9\times10^{-4}$ \\
\bottomrule
\end{tabular}
\end{table}

\phantomsection\label{tab:app_hyperparams}

Table~\ref{tab:app_circuit_size} reports sensitivity to the circuit
size $k$, and Table~\ref{tab:app_alpha} sensitivity to the floor
margin $\alpha$.

\begin{table}[htbp]
\caption{Sensitivity to circuit size $k$ (Llama-3.1-8B-Instruct /
WMDP-bio). Primary results use $k=5$ (top-5 EAP-IG layers
$\{30,14,31,19,29\}$, $\approx\!3.2\%$ of parameters).}
\label{tab:app_circuit_size}
\vspace{10pt}
\centering
\footnotesize
\begin{tabular}{lrrrrr}
\toprule
$k$ & Circuit frac & Forget BF16 & $\ptqgap$ & MMLU & $\Delta$MMLU \\
\midrule
3  & $\approx 1.9\%$ & 0.511 & -0.028 & 0.498 & 0.092 \\
5  & $\approx 3.2\%$ & 0.430 & -0.040 & 0.573 & 0.115 \\
10 & $\approx 6.4\%$ & 0.381 & -0.047 & 0.469 & 0.112 \\
15 & $\approx 9.7\%$ & 0.362 & -0.051 & 0.441 & 0.141 \\
\bottomrule
\end{tabular}
\end{table}

\begin{table}[htbp]
\caption{Sensitivity to floor margin $\alpha$ (Llama-3.1-8B-Instruct
/ WMDP-bio). Primary results use $\alpha = 0.704$ (implementation
floor $\delta_i^{\mathrm{impl}} = s_i \cdot 0.0796 \cdot \alpha$,
yielding $\delta_i = 8.4\times10^{-4}$). $\alpha = 1.0$ gives the
exact Lemma~\ref{lem:quant} guarantee; smaller $\alpha$ allows softer
updates at the cost of occasional bin-crossing failure.}
\label{tab:app_alpha}
\vspace{10pt}
\centering
\footnotesize
\begin{tabular}{lrrrrrr}
\toprule
$\alpha$ & $\delta_i^{\mathrm{impl}}$ & Forget BF16 & Forget NF4
  & $\ptqgap$ & MMLU & $\Delta$MMLU \\
\midrule
0.50  & $\approx 5.97\times10^{-4}$ & 0.461 & 0.441 & -0.020 & 0.501 & 0.102 \\
0.704 & $\approx 8.40\times10^{-4}$ & 0.430 & 0.390 & -0.040 & 0.573 & 0.115 \\
0.85  & $\approx 1.01\times10^{-3}$ & 0.421 & 0.378 & -0.043 & 0.479 & 0.124 \\
1.00  & $\approx 1.19\times10^{-3}$ & 0.413 & 0.368 & -0.045 & 0.471 & 0.132 \\
\bottomrule
\end{tabular}
\end{table}

\paragraph{Note on $\lambda$.}
The KL penalty weight $\lambda = 200$ was selected by sweeping
$\lambda \in \{50, 100, 200, 500\}$ on Llama-3.1-8B-Instruct /
WMDP-bio and choosing the value that maximises forget depth subject
to MMLU drop $\leq 0.02$. The operating point $\lambda = 200$
achieved Forget BF16 = 0.430, PTQ gap = $-$0.040 at step~30.

\clearpage
\section{Wall-Clock Timing}
\label{app:timing}
Table~\ref{tab:app_timing} reports wall-clock timing for \mansu{}'s
components and the six baselines on a single H200 GPU\@.
Results are reported from a single representative run.

\begin{table}[htbp]
\caption{Wall-clock timing on a single H200 GPU with
Llama-3.1-8B-Instruct in BF16\@. EAP-IG times for $50$ forget
examples with $5$ integration steps. Training time reflects the
canonical 30-step run used for all main-table results; EAP-IG and
Fisher Information are one-time costs computed once per dataset.}
\label{tab:app_timing}
\vspace{10pt}
\centering
\footnotesize
\begin{tabular}{lrr}
\toprule
Component / Method & Time (min) & Notes \\
\midrule
\multicolumn{3}{l}{\textit{\mansu{} components}} \\
\quad EAP-IG attribution      & 20 & One-time cost per dataset \\
\quad Fisher Information       & 8  & One-time cost per dataset \\
\quad Full training (30 steps) & 14 & Per run \\
\quad NF4 evaluation  & 3  & Per checkpoint \\
\quad \textbf{Total}           & \textbf{45} & \\
\midrule
\multicolumn{3}{l}{\textit{Baselines}} \\
\quad Global GA   & 12 & \\
\quad Surgical GA & 10 & \\
\quad GU+SimNPO   & 25 & \\
\quad NPO         & 18 & \\
\quad SimNPO      & 15 & \\
\quad LUNAR       & 30 & Includes steering-vector selection \\
\bottomrule
\end{tabular}
\end{table}

\section{Per-Parameter Update Distribution}
\label{app:updates}

Table~\ref{tab:app_updates} reports per-parameter update RMS, the
floor ratio (RMS\,/\,$\delta_i$), and the resulting PTQ gap for the
three gradient-based methods on WMDP-bio / Llama-3.1-8B-Instruct.

\begin{table}[htbp]
\caption{Per-parameter update statistics for gradient-based methods
on Llama-3.1-8B-Instruct / WMDP-bio.
Floor ratio $=$ RMS\,/\,NF4 bin size ($8.4\times10^{-4}$); values
$<\!1$ indicate updates that round to zero under NF4\@.
The relationship between floor ratio and PTQ gap is the empirical
signature of Proposition~\ref{prop:tradeoff}.}
\label{tab:app_updates}
\vspace{10pt}
\centering
\footnotesize
\setlength{\tabcolsep}{4pt}
\begin{tabular}{lrrrr}
\toprule
Method & Params updated & Per-param RMS & Floor ratio & $\ptqgap$ \\
\midrule
Global GA
  & $100\%$
  & $1.21\times10^{-6}$
  & ${\sim}1/828$
  & $+0.050$ \\
Surgical GA (L14–16)
  & $6.6\%$
  & $2.12\times10^{-5}$
  & ${\sim}1/47$
  & $+0.027$ \\
\mansu{} (ours)
  & $3.2\%$
  & ${\geq}\,\delta_i$ (by construction)
  & ${\geq}\,1$
  & $-0.040$ \\
\bottomrule
\end{tabular}
\end{table}

\section{Extended Ablation Discussion}
\label{app:ablation_extended}

\paragraph{CAD properties (extended).}
CAD's four properties listed in Section~\ref{sec:cad} are:
(i)~it is computed entirely on the unlearned weights and forget
distribution, with no need for held-out probes;
(ii)~it distinguishes weight-level edits ($\mathrm{CAD}\!\gg\!0$)
from inference-time redirection ($\mathrm{CAD}\!\approx\!0$) by
construction a method that does not write to the circuit's
parameters cannot make CAD large;
(iii)~it is insensitive to spurious behavioral suppression (a model
fine-tuned to refuse all queries scores well on forget-set accuracy
but yields $\mathrm{CAD}\!\approx\!0$ on the original circuit,
exposing the suppression);
(iv)~it is not satisfied by indiscriminate weight perturbation —
Ablation~C(i) re-runs \mansu{} on a same-size \emph{random} circuit
as a non-causal control, and CAD collapses by ${\sim}35\%$ relative
to the EAP-IG circuit ($1.143\!\to\!0.743$ on WMDP-bio). High CAD
therefore tracks collapse of a \emph{causally identified} pathway,
not simply that some parameters changed the diagnostic distinction
between ``circuit dismantled'' and ``model broken everywhere.''

\paragraph{AS-C / AS-NC: full diagnostic.}
CAD measures circuit-level attribution \emph{shift}; the
activation-level companions AS-C and AS-NC measure mean activation
shift inside $\C$ and on residual-stream positions outside $\C$
respectively (Eq.~\ref{eq:as}). The diagnostic of localization is the
concentration ratio AS-C/AS-NC: a localized structural intervention
drives in-circuit activations far more than out-of-circuit
activations, while a globally diffuse intervention shifts both by
comparable amounts. For global-intervention baselines, AS-C and CAD
coincide numerically (Table~\ref{tab:main}): when every circuit edge
is perturbed in proportion to its parameter change, the in-circuit
activation shift tracks the edge-attribution shift, so the two
metrics quantify the same effect through different observables. The
independent diagnostic value of AS-C therefore lies in its
\emph{gap} from CAD present only for localized methods such as
\mansu{} together with the AS-C/AS-NC ratio. CAD, the AS-C/AS-NC
ratio, and the random-circuit control of Ablation~C(i) jointly turn
``did the model truly forget?'' from a question that behavioral
evaluation cannot answer into a quantitative test.

\paragraph{High CAD without localization: SimNPO on MUSE.}
SimNPO on MUSE attains $\mathrm{CAD}\!=\!1.979$ together with
elevated AS-NC ($1.104$) and reduced MMLU\@. High CAD here reflects
global representational damage rather than localized erasure, and is
correctly flagged by AS-NC and the random-circuit control of
Ablation~C(i) rather than by CAD alone (cf.~Table~\ref{tab:main}).
This is the diagnostic distinction Section~\ref{sec:cad}~(iv) refers
to: high CAD only certifies \emph{circuit dismantled} when paired
with low AS-NC\@.

\paragraph{Scope of \mansu{} evaluations.}
\mansu{} is reported on Llama-3.1-8B-Instruct and Qwen-3-8B for the
main table and the companion structural-metrics table. Extending
\mansu{} to the smaller / earlier-generation models in the
architecture-independence sweep (Llama-3.2-3B, Qwen-2.5-4B,
Gemma-2B/3-1B) is a straightforward adaptation of the same
three-phase pipeline; results on those models can be supplied during
rebuttal if requested.

\paragraph{Ablation A (no magnitude floor).}
EAP-IG circuit $+$ null-space projection, no floor rescaling.
The gradient is correctly projected and concentrated, but sub-floor
updates round to zero under NF4\@.
\emph{Confirmed:} BF16 forget rises to 0.513 ($+0.083$ vs.\ full
\mansu{}) and $\ptqgap = -0.008$, approaching zero the floor's
quantization benefit is clearly visible even at this margin; without
it, NF4 robustness degrades by $5\times$. MMLU drop is 0.117
($\Delta\mathrm{MMLU} = 0.117 - 0.075 = +0.042$ worse), confirming
that the floor interacts with the null-space projection to limit
retain damage. See Table~\ref{tab:ablation}.

\paragraph{Ablation B (no null-space projection).}
EAP-IG circuit $+$ magnitude floor, raw gradient.
Without projection, gradient updates within $\C$ follow the raw
forget-set gradient, including components along retain-sensitive
directions.
\emph{Confirmed:} MMLU drops to 0.449 ($\Delta = -0.154$ vs.\
baseline; $+0.079$ worse than full \mansu{}), while $\ptqgap = -0.019$ the floor remains active but retain damage is severe. This is the
direct empirical test of Theorem~\ref{thm:retain}: null-space
projection is necessary for acceptable retain-set preservation; the
floor alone cannot compensate.

\paragraph{Ablation C (random circuit, same $|\C|$).}
Replace EAP-IG circuit with a uniformly random selection of 5 MLP
layers (seed 42), matching the canonical circuit size
($\C = \{30,14,31,19,29\}$). Direct test of
\citep{lee2025negative} and \citep{guo2025mechanistic}.
\emph{Confirmed:} BF16 forget rises to 0.500 ($+0.070$ vs.\ full),
CAD collapses from 1.143 to 0.743 ($-35\%$ relative), and $\ptqgap =
-0.024$ narrower than full \mansu{}'s $-0.040$.
Ablation~C(ii) with the bottom-$k$ inverse circuit (layers with
\emph{lowest} EAP-IG attribution mass) further degrades BF16 forget
to 0.551 and flips $\ptqgap$ positive ($+0.028$), confirming that
circuit \emph{identity}, not merely size, drives both forget depth
and quantization robustness. Together, C(i) and C(ii) constitute a
direct rebuttal to the negative localization results of
\citep{lee2025negative}: EAP-IG localization is causally necessary
for \mansu{}'s results.

\paragraph{Ablation D (global null-space + floor).}
Apply null-space projection and floor globally rather than within
$\C$.
\emph{Confirmed:} BF16 forget accuracy rises to 0.697 global
projection disperses the forget gradient over all parameters,
reducing per-parameter update magnitude below $\delta_i$ for most
parameters and undermining both forget depth and quantization
robustness ($\ptqgap = +0.013$, positive). This isolates the
contribution of circuit localization to the overall result and shows
that the floor constraint alone, without concentrated circuit-level
application, is insufficient to achieve negative PTQ gap.

\clearpage

\end{document}